\documentclass{article}

\usepackage[preprint]{neurips_2026}

\usepackage[utf8]{inputenc} % allow utf-8 input
\usepackage[T1]{fontenc}    % use 8-bit T1 fonts
\usepackage{hyperref}       % hyperlinks
\usepackage{url}            % simple URL typesetting
\usepackage{booktabs}       % professional-quality tables
\usepackage{amsfonts}       % blackboard math symbols
\usepackage{nicefrac}       % compact symbols for 1/2, etc.
\usepackage{microtype}      % microtypography
\usepackage{xcolor}         % colors

\usepackage[disable,textsize=tiny]{todonotes}
\RequirePackage{algorithm}
\RequirePackage{algorithmic}
\usepackage{wrapfig}

\usepackage{hyperref}       % hyperlinks
\usepackage{url}            % simple URL typesetting
\usepackage{booktabs}       % professional-quality tables
\usepackage{nicefrac}       % compact symbols for 1/2, etc.
\usepackage{microtype}      % microtypography
\usepackage{xcolor}         % colors

\usepackage{amsmath,amsfonts,amssymb,amsthm,bm}
\usepackage{graphicx}
\usepackage{enumitem}

\usepackage[normalem]{ulem}
\usepackage{multirow}
\usepackage{makecell}

\usepackage{cleveref}

\usepackage{xargs}

\definecolor{blush}{rgb}{0.87, 0.36, 0.51}
\newcommand{\ak}[1]{\todo[color=blush,size=\footnotesize]{ \textbf{AK:}  #1}}

\newcommand{\E}{\mathbb{E}}

\newcommand{\bP}{\mathbb{P}}
\newcommand{\R}{\mathbb{R}}
\newcommand{\bS}{\mathbb{S}}

\newcommand{\cD}{\mathcal{D}}

\newcommand{\cM}{\mathcal{M}}
\newcommand{\cN}{\mathcal{N}}
\newcommand{\cP}{\mathcal{P}}
\newcommand{\cU}{\mathcal{U}}
\newcommand{\cX}{\mathcal{X}}
\newcommand{\cY}{\mathcal{Y}}

\newcommand{\bOne}{\bm 1}

\newcommand{\veps}{\varepsilon}

\newcommand{\W}{\mathop{\mathrm{W}}\nolimits}
\newcommand{\SW}{\mathop{\mathrm{SW}}\nolimits}

\newcommand{\dd}{\;\mathrm{d}}

\newcommand{\spt}{\mathrm{spt}}
\newcommand{\Lap}{\mathrm{Lap}}
\newcommand{\Beta}{\mathrm{Beta}}

\newcommand{\setcond}{\;|\;}

\newcommand{\sca}[2]{\langle#1|#2\rangle}

\newcommand{\Rescale}{\mathop{\mathrm{Rescale}}\nolimits}

\theoremstyle{plain}
\newtheorem{theorem}{Theorem}[section]
\newtheorem{proposition}[theorem]{Proposition}
\newtheorem{lemma}[theorem]{Lemma}

\theoremstyle{definition}
\newtheorem{definition}[theorem]{Definition}

\theoremstyle{remark}
\newtheorem{remark}[theorem]{Remark}

\title{Highly Data Parallelizable Estimation of the Sliced-Wasserstein Distance Using Cumulative Distribution Functions}

\author{%
    Christophe Vauthier \\
    Laboratoire de Mathématiques d'Orsay, Université Paris-Saclay \\
    Gif-sur-Yvette, France \\
    \texttt{christophe.vauthier@universite-paris-saclay.fr} \\
    \And
    Quentin Mérigot \\
    Laboratoire de Mathématiques d'Orsay, Université Paris-Saclay \\
    Gif-sur-Yvette, France \\
    \texttt{quentin.merigot@universite-paris-saclay.fr} \\
    \And
    Anna Korba \\
    Centre de recherche en économie et statistique, ENSAE \\
    Palaiseau, France \\
    \texttt{anna.korba@ensae.fr} \\
}

\begin{document}

\maketitle

\begin{abstract}
  The Sliced Wasserstein (SW) distance has emerged as a computationally attractive alternative to the Wasserstein distance by leveraging one-dimensional optimal transport along random projections. Standard estimators of the SW distance rely on Monte Carlo averages of one-dimensional Wasserstein distances computed via quantile functions, which require sorting projected samples and access to full datasets. In this work, we introduce a new class of estimators for the Sliced Wasserstein distance based on cumulative distribution functions (CDFs) of projected measures, that avoid sorting and scale via massive dataset parallelism. This class includes several estimators, some of them being indexed by hyperparameters controlling their variance or smoothness. We show that they are especially well suited to scenarios in which CDFs are more tractable than quantile functions, such as mixtures of Gaussians, and moreover that they are also naturally compatible with federated learning, since CDFs of projected data can be computed and aggregated locally without requiring the exchange of raw samples.
\end{abstract}

\section{Introduction}

The Sliced Wasserstein (SW) distance~\citep{rabin2012wasserstein,bonneel2015sliced} has emerged as a widely used alternative to the Wasserstein distance for comparing high-dimensional probability measures. By projecting measures onto one-dimensional subspaces and averaging the resulting one-dimensional Wasserstein distances, the SW distance exploits the closed-form structure of optimal transport on the real line while significantly reducing computational complexity. As a result, SW distances have been successfully applied in a broad range of machine learning tasks, including generative modeling~\citep{liutkus2019sliced,kolouri2018sliced,nguyen2021distributional,nguyen2024energy}, clustering~\citep{kolouri2018sliced}, variational inference~\citep{yi2023sliced}, domain adaptation~\citep{lee2019sliced}, approximate Bayesian computation~\citep{nadjahi2020approximate}.

In practice, the SW distance is often estimated using Monte Carlo sampling over projection directions. For each sampled direction, the one-dimensional Wasserstein distance between projected measures is computed exactly via quantile functions, which for empirical measures amounts to sorting the projected samples. This procedure yields an unbiased estimator of the SW distance and scales as $\mathcal{O}(LN\log(N))$ for $L$ projections and  $N$ samples. While this complexity is substantially lower than that of high-dimensional optimal transport, quantile-based estimators of the SW distance suffer from structural limitations. The sorting operation requires access to all projected samples, which makes standard SW estimators poorly suited to distributed and federated learning settings, where data are split across multiple clients and cannot be centrally aggregated. Moreover, although different projections can be processed independently, the need to compute global order statistics on all samples limits parallelism. These limitations motivate the search for alternative estimators of the SW distance that preserve its favorable properties while relaxing its computational and structural constraints.

A number of extensions of the Sliced Wasserstein distance have been proposed to improve flexibility or computational burden, including max-sliced Wasserstein distances~\citep{deshpande2019max}, non linear projections~\citep{kolouri2019generalized}, learned or adaptive projections~\citep{nguyen2021distributional,ohana2023shedding,nguyen2023markovian,dai2020sliced}, orthogonal projections~\citep{rowland2019orthogonal} and hierarchical slicing schemes~\citep{nguyen2022hierarchical}.
Other works have proposed generalizations of SW distances where projections on lines are replaced by projections on a more advanced structure called tree systems~\citep{le2019tree,tran2025distance}. 
However, all these approaches still rely on quantile functions and therefore retain the need to compute global order statistics of the projected samples, that require access to the full dataset.

In this work, we propose a novel approach to estimating the Sliced Wasserstein distance, based on cumulative distribution functions (CDFs) of projected measures rather than quantile functions. Our construction builds on integral representations of one-dimensional Wasserstein distances, which express Wasserstein distances of orders one and two directly in terms of CDFs. Leveraging these representations, we derive unbiased Monte Carlo estimators of the SW distance that require only pointwise evaluations of projected CDFs at randomly sampled locations, completely avoiding sorting operations. This CDF-based formulation yields estimators with several desirable properties. From a computational standpoint, evaluating a CDF reduces to a simple sum over projected samples and can be implemented in linear time, leading to estimators with complexity 
$O(LNd)$. Furthermore, the CDF of a dataset is naturally expressed as a sum of CDFs of subdatasets, making the proposed estimators massively data parallelizable. 

Yet, we empirically analyze the variance of the resulting estimators and show that a naive Monte-Carlo approximation may lead to poor accuracy for second-order SW distances. To address this issue, we introduce an importance sampling strategy that significantly improves estimation accuracy.
A further challenge of CDF-based estimators is that empirical CDFs are not differentiable with respect to sample locations, which limits their direct use in gradient-based learning. To overcome this limitation, we introduce smooth approximations of CDFs that yield differentiable estimators of the Sliced-Wasserstein distance. These smoothed estimators converge to their original counterparts as the smoothing parameter vanishes, while enabling backpropagation in applications involving neural network training, such as generative modeling.
Finally, the proposed estimators are particularly well suited to federated learning settings. Since CDFs of projected measures can be computed locally and aggregated by simple summation, the SW distance can be estimated without sharing raw data or projected samples. This property enables efficient and private computation of SW distances in distributed environments and allows SW distances to be integrated into federated learning algorithms.

This paper is organized as follows. In \Cref{sec:background}, we provide background material on optimal transport, one-dimensional Wasserstein distances, and the Sliced Wasserstein distance. 
\Cref{sec:cdf_estimators} introduces the proposed CDF-based estimators of the Sliced Wasserstein distance, analyzes their computational and statistical properties, and presents variance reduction and smoothing techniques. \Cref{sec:applications} discusses practical applications of the proposed estimators, with a particular focus on federated learning settings or Gaussian mixtures. \Cref{sec:experiments} presents our experimental results on some generative modeling tasks, Gaussian mixture models and federated learning problems.

\section{Background}\label{sec:background}

\paragraph{Notations} We denote $\cP(\R^d)$ the set of Borel probability measures on $\R^d$, and, for every $p \geq 1$, we denote $\cP_p(\R^d)$ the set of probability measures with finite $p$-th order moment. For every $\mu \in \cP(\R^d)$ and every Borel measurable map $T : \R^d \to \R^m$, we denote $T_\#\mu \in \cP(\R^m)$ the pushforward of $\mu$ by $T$, defined by $T_\#\mu(A) := \mu(T^{-1}(A))$ for every Borel set $A \subseteq \R^m$. The support of a measure $\mu \in \cP(\R^d)$, which is the complement of the largest open set of zero $\mu$-measure, is denoted $\spt(\mu)$. Given $X$ a compact differential manifold, we will denote $\cU(X)$ the uniform probability distribution on $X$. We note $S_d(\R)$ the space of $d \times d$ symmetric matrices, and $S_d^{++}(\R)$ the space of symmetric positive definite matrices. For every $m \in \R^d$ and $\Sigma \in S_d^{++}(\R)$, we denote $\cN(m,\Sigma)$ the normal distribution with mean vector $m$ and covariance matrix $\Sigma$. For every $x \in \R$ we denote $(x)^+ := \max(0,x)$ the positive part of $x$.

\paragraph{Wasserstein distance} Given $p \geq 1$ and two probability measures $\mu,\nu \in \cP_p(\R^d)$, the Wasserstein distance of order $p$ between $\mu$ and $\nu$ is defined by
\begin{equation} \label{eq:wasserstein_distance}
    \W_p(\mu,\nu) := \left(\min_{\gamma \in \Pi(\mu,\nu)} \int |x-y|^p \dd\gamma(x,y)\right)^{1/p}
\end{equation}
where $\Pi(\mu,\nu)$ is the set of  $\gamma \in \cP(\R^d \times \R^d)$ such that $\pi_{1\#}\gamma = \mu$ and $\pi_{2\#}\gamma = \nu$, where $\pi_1,\pi_2 : \R^d \times \R^d \to \R^d$ are the projections on the first and second coordinate, and $|\cdot|$ denotes the Euclidean distance. When $\mu$ and $\nu$ are two point clouds of the form $\mu := \frac 1N \sum_{i=1}^N \delta_{x_i}$ and $\nu := \frac 1N \sum_{i=1}^N \delta_{y_i}$, \eqref{eq:wasserstein_distance} can be reformulated as a particular case of a linear program, and dedicated algorithms allow to compute the Wasserstein distance in time complexity $O(N^{2.5}\log(N))$ \citep{burkard2012assignment}.

\paragraph{1D Wasserstein distance} For a probability measure $\mu \in \cP(\R)$, we note $F_\mu : \R \mapsto [0,1]$ its cumulative distribution function, defined by $F_\mu(x) := \mu((-\infty, x])$. It is nondecreasing and right continuous. We also define its quantile function $F_\mu^{-1}$ to be the pseudo-inverse of $F_\mu$, defined by $F^{-1}_\mu(x) := \inf \{t \in \R \setcond F_\mu(t) \geq x \}$ for every $x \in [0,1]$. It is well-known (see \citep[Chapter 2]{santambrogio2015optimal}) that in dimension $d = 1$, the Wasserstein distance between two probability measures can be explicitly expressed as the $L^p$ distance between their quantile functions: for every $\mu, \nu \in \cP(\R)$,
\begin{equation} \label{eq:wass_1d_expression_quantile}
    \W_p^p(\mu, \nu) = \int_0^1 |F_\mu^{-1}(x) - F_\nu^{-1}(x)|^p \dd x
\end{equation}
This property allows for efficient computation of the Wasserstein distance in dimension 1. Indeed, if $\mu,\nu$ are point clouds of the form $\mu := \frac 1N \sum_{i=1}^N \delta_{x_i}$ and $\nu := \frac 1N \sum_{i=1}^N \delta_{y_i}$, the distance $\W_p(\mu,\nu)$ can be computed in time $\mathcal{O}(N\log(N))$, by sorting the $x_i$'s and the $y_j$'s: letting $\sigma,\tau$ be two permutations of $\{1,\ldots,N\}$ such that $x_{\sigma(1)} \leq \ldots \leq x_{\sigma(N)}$ and $y_{\tau(1)} \leq \ldots \leq y_{\tau(N)}$, we have
\begin{equation}
    \W_p^p(\mu,\nu) = \frac 1N \sum_{i=1}^N |x_{\sigma(i)} - y_{\tau(i)}|^p.
\end{equation}

\paragraph{Sliced Wasserstein distance} Let $p \geq 1$ and $\mu, \nu \in \cP_p(\R^d)$. The Sliced Wasserstein distance of order $p$ between $\mu$ and $\nu$ is
\begin{equation}
    \SW_p(\mu,\nu) := \left(\int_{\bS^{d-1}} \W_p^p(\mu_\theta, \nu_\theta) \dd\theta \right)^{1/p}
\end{equation}
where $\dd\theta$ is the uniform measure on the unit sphere $\bS^{d-1}$, and where $\mu_\theta := P_{\theta\#}\mu$ and $\nu_\theta := P_{\theta\#}\nu$, with $P_\theta$ the projection $x \in \R^d \mapsto \sca{x}{\theta} \in \R$ on the direction $\theta$. The Sliced Wasserstein distance between two point clouds $\mu := \frac 1N \sum_{i=1}^N \delta_{x_i}$ and $\nu := \frac 1N \sum_{i=1}^N  \delta_{y_i}$ is usually computed using the Monte-Carlo estimator
\begin{equation} \label{eq:sw_quantile_estimator}
    \widetilde{\SW}_{p,L}^p(\mu,\nu) := \frac{1}{L} \sum_{l=1}^L \W_p^p(\mu_{\theta_l}, \nu_{\theta_l})
\end{equation}
where $L > 0$ and $\theta_1,\ldots,\theta_L$ are i.i.d. random variables with $\theta_i \sim \cU(\bS^{d-1})$. This gives an unbiased estimator of $\SW_p^p(\mu,\nu)$, which can be computed in time $\mathcal{O}(L(N\log(N) + Nd))$.

\section{The CDF estimators of the Sliced Wasserstein} \label{sec:cdf_estimators}

\paragraph{The estimators for empirical measures} We describe here how to construct our alternative and novel estimators for the Sliced-Wasserstein distance. Our method is based on the following two observations made in \citep{bobkovonedimensional2019}:

\begin{theorem} \label{th:bobkov_ledoux_sw1} \citep[Theorem 2.9]{bobkovonedimensional2019}
    Let $\mu, \nu \in \cP_1(\R)$, then we have
    \begin{equation}
        \W_1(\mu,\nu) = \int_{-\infty}^\infty |F_\mu(x) - F_\nu(x)| \dd x
    \end{equation}
\end{theorem}

\begin{theorem} \label{th:bobkov_ledoux_sw2} \citep[Theorem 2.11]{bobkovonedimensional2019}
    Let $\mu, \nu \in \cP_2(\R)$, then we have
    \begin{equation}
        \W_2^2(\mu,\nu) = 2\iint_{x \leq y} (F_\mu(x) - F_\nu(y))^+ + (F_\nu(x) - F_\mu(y))^+ \dd x \dd y.
    \end{equation}
\end{theorem}
These theorems immediately give alternative expressions for the Sliced-Wasserstein distances of order 1 and 2. For every $\mu, \nu \in \cP_1(\R^d)$, we have
\begin{equation} \label{eq:sw1_alternate}
    \SW_1(\mu,\nu) = \int_{\bS^{d-1}} \int_{-\infty}^\infty |F_{\mu_\theta}(x) - F_{\nu_\theta}(x)| \dd x \dd \theta,
\end{equation}
and for every $\mu, \nu \in \cP_2(\R^d)$, it holds
\begin{equation} \label{eq:sw2_alternate}
    \SW_2^2(\mu,\nu) = 2\int_{\bS^{d-1}} \iint_{x \leq y} (F_{\mu_\theta}(x) - F_{\nu_\theta}(y))^+ + (F_{\nu_\theta}(x) - F_{\mu_\theta}(y))^+ \dd x \dd y \dd \theta.
\end{equation}
Furthermore, if for every $\theta \in \bS^{d-1}$, we let $[a_\theta,b_\theta]$ be any interval containing the supports of $\mu_\theta$ and $\nu_\theta$, then the inner integral in \eqref{eq:sw1_alternate} can be restricted to $[a_\theta,b_\theta]$ and the inner double integral in \eqref{eq:sw2_alternate} can be restricted to $a_\theta \leq x \leq y \leq b_\theta$. Using these alternative expressions, we construct estimators for the $\SW_1$ and $\SW_2$ between two discrete measures $\mu := \sum_{i=1}^N \alpha_i \delta_{x_i}$ and $\nu := \sum_{j=1}^M \beta_j \delta_{y_j}$. For every $\theta \in \bS^{d-1}$, let $m^-_\theta := \min(\spt(\mu_\theta) \cup \spt(\nu_\theta))$ and $m^+_\theta := \max(\spt(\mu_\theta) \cup \spt(\nu_\theta))$ (so $[m^-_\theta,m^+_\theta]$ is the smallest interval containing all the projections $\sca{x_i}{\theta}$ and $\sca{y_j}{\theta}$). Then, up to a change of variable $\tilde{x} = m_\theta^- + r_\theta x$, with $r_\theta = m_\theta^+ - m_\theta^-$, the inner integral in \eqref{eq:sw1_alternate} and the inner double integral in \eqref{eq:sw2_alternate} can be restricted to respectively $[0,1]$ and the triangle $T = \{(x,y) \in [0,1]^2 \setcond x \leq y\}$. This will make sampling easier in the definitions of the estimators.

\begin{definition} \label{def:sw1_alt_estimator}
    Let $L > 0$, and let $(\theta_1,U_1),\ldots,(\theta_L,U_L)$ be i.i.d. variables with $(\theta_1,U_1) \sim \mathcal{U}(\bS^{d-1} \times [0,1])$. We define the estimator $\widehat{\SW}_{1,L}(\mu,\nu)$ of $\SW_1(\mu,\nu)$ by
    \begin{equation} \label{eq:sw1_alt_estimator}
        \widehat{\SW}_{1,L}(\mu,\nu) := \frac 1L \sum_{l=1}^L r_{\theta_l} \Delta_{\mu,\nu}^{(1)}(\theta_l,\tilde{U}_l)
    \end{equation}
    where $\Delta_{\mu,\nu}^{(1)}(\theta,u) := |F_{\mu_\theta}(u) - F_{\nu_\theta}(u)|$ for every $u \in \R$, and $\tilde{U}_l := \Rescale(U_l, m^-_{\theta_l}, m^+_{\theta_l})$ for every $l = 1,\ldots,L$, with $\Rescale(x,a,b) := a + (b-a)x$ for every $x,a,b \in \R$.
\end{definition}

\begin{definition} \label{def:sw2_alt_estimator}
    Let $(\theta_1,U_1,V_1),\ldots,(\theta_L,U_L,V_L)$ be i.i.d. variables with $(\theta_1,U_1,V_1) \sim \mathcal{U}(\bS^{d-1} \times T)$. We define the estimator $\widehat{\SW}_{2,L}^2(\mu,\nu)$ of $\SW_2^2(\mu,\nu)$ by
    \begin{equation} \label{eq:sw2_alt_estimator}
        \widehat{\SW}_{2,L}^2(\mu,\nu) := \frac 1L \sum_{l=1}^L r^2_{\theta_l} \Delta^{(2)}_{\mu,\nu}(\theta_l, \tilde{U}_l, \tilde{V}_l)
    \end{equation}
    where $\Delta_{\mu,\nu}^{(2)}(\theta,u,v) := (F_{\mu_\theta}(u) - F_{\nu_\theta}(v))^+ + (F_{\nu_\theta}(u) - F_{\mu_\theta}(v))^+$ for every $u,v \in \R$ and $\theta \in \bS^{d-1}$, and again $\tilde{U}_l := \Rescale(U_l, m^-_{\theta_l}, m^+_{\theta_l})$ and $\tilde{V}_l := \Rescale(V_l, m^-_{\theta_l}, m^+_{\theta_l})$.
\end{definition}
It is easy to check using \eqref{eq:sw1_alternate} and \eqref{eq:sw2_alternate} that these estimators are unbiased, i.e. that 
\begin{equation}
    \E[\widehat{\SW}_{p,L}^p(\mu,\nu)] = \SW_p^p(\mu,\nu), \quad p \in \{1,2\}.
\end{equation}
Computing the projections $\sca{x_i}{\theta_l}$ and $\sca{y_j}{\theta_l}$ can be done in time $\mathcal{O}(L(N+M)d)$. Moreover, for any $a \in \R$, it holds
$F_{\mu_\theta}(a) = \sum_{i=1}^N \alpha_i \bOne_{\sca{x_i}{\theta} \leq a}$ and $F_{\nu_\theta}(a) = \sum_{j=1}^M \beta_j \bOne_{\sca{y_j}{\theta} \leq a}$, so to compute the cumulative distribution functions $F_{\mu_\theta}(a)$ and $F_{\nu_\theta}(a)$, we only need to sum the $\alpha_i$ and $\beta_j$ of the points $x_i$ and $y_j$ whose projection on the line $\theta$ is less than or equal to $a$, which can be done in time respectively $\mathcal{O}(N)$ and $\mathcal{O}(M)$ (provided the $\sca{x_i}{\theta}$ and $\sca{y_j}{\theta}$ have already been computed). Therefore the time complexity of computing these estimators is of the order of $\mathcal{O}(L(M+N)d)$. Note that, while this implies that for a fixed $L$ the estimators $\widehat{\SW}_{1,L}$ and $\widehat{\SW}^2_{2,L}$ will be faster to compute than the usual Monte-Carlo estimators $\widetilde{\SW}_{1,L}$, $\widetilde{\SW}^2_{2,L}$, which have complexity $O(L(N\log(Nd) + Nd))$, this comes at the cost of higher variance, as we expect the integrands $\Delta^{(1)}_{\mu,\nu}$ and $\Delta^{(2)}_{\mu,\nu}$ to have higher variance than the integrand $\W_p^p(\mu_\theta,\nu_\theta)$. We refer to \Cref{appendix:sec:sec:sec:exp_time_vs_acc} for an empirical analysis of the time-accuracy tradeoff between the various estimators.

\begin{wrapfigure}{l}{0.5\textwidth}
    \centering
    \begin{center}
        \includegraphics[width=\linewidth]{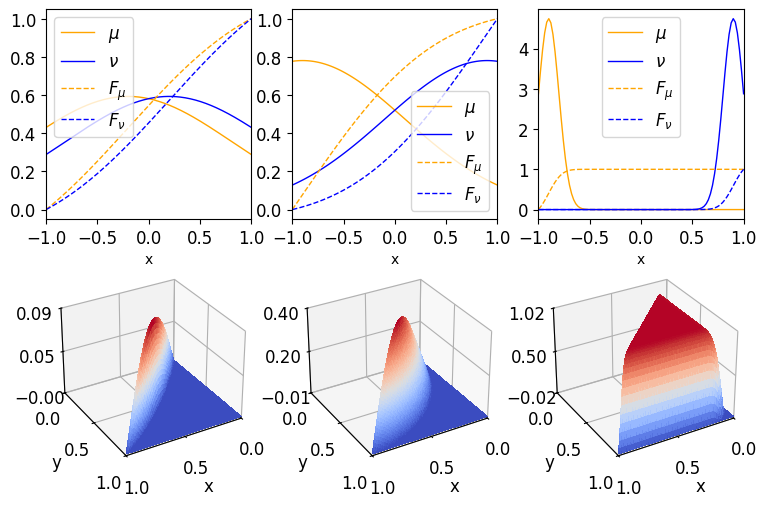}
    \end{center}
    \caption{Plots of $(F_\mu(x)-F_\nu(y))^+$ on the triangle $0 \leq x \leq y \leq 1$ (bottom) and of the probability density functions and CDFs of $\mu,\nu$ (top) for various pairs $(\mu,\nu)$. The measures $\mu$ and $\nu$ are obtained by truncating normal distributions $\cN(m_-,\sigma)$ and $\cN(m_+,\sigma)$ to $[-1,1]$. Left: $m_\pm = \pm 0.2$ and $\sigma = 1$. Middle: $m_\pm = \pm 0.9$ and $\sigma = 1$. Right: $m_\pm = \pm 0.9$ and $\sigma = 0.1$.}
    \label{fig:importance_sampling}
\end{wrapfigure}

\paragraph{Variance minimization via importance sampling} One drawback of the estimator $\widehat{\SW}_{2,L}^2$, which comes from the formula \eqref{eq:sw2_alternate} on which it is based, is that for $\mu, \nu \in \cP_2(\R)$, the function $(x,y) \to (F_\mu(x) - F_\nu(y))^+$ is nondecreasing in $x$ and nonincreasing in $y$, so that if $\mu$ and $\nu$ are ``close", it is highly likely that we will have $F_\mu(x) \leq F_\nu(y)$ whenever the gap $y-x > 0$ is large. This is illustrated in \Cref{fig:importance_sampling} which shows that, the more the measures $\mu$ and $\nu$ are close, the more the subset of $T$ on which $(F_\mu(x) - F_\nu(y))^+$ is nonzero is concentrated near the diagonal $x = y$. Therefore, in this situation, computing $\widehat{\SW}_{2,L}^2$ by sampling the triangle $(U_l, V_l) \in T$ uniformly will likely lead to many zero terms, hence to a high variance. One way of compensating for this problem is to instead sample the triangle $T$ using a probability distribution $\eta$ which gives more weight to pairs $(x,y) \in T$ with $y-x$ small.\\

\begin{definition} \label{def:sw2_alt_estimator_k}
    Let $\eta \in \cP(T)$ be with positive density $f$ (that is $\dd\eta(x,y) = f(x,y) \dd x \dd y$). Let $(\theta_1,U_1,V_1)$, ...,$(\theta_L,U_L,V_L)$ be i.i.d. random variables with $(\theta_1,U_1,V_1) \sim \mathcal{U}(\bS^{d-1}) \times \eta$. We define the estimator
    \begin{equation} \label{eq:sw2_alt_estimator_2}
        \widehat{\SW}_{2,L,\eta}^2(\mu,\nu) := \frac 1L \sum_{l=1}^L \frac{2r^2_{\theta_l}}{f(U_l,V_l)} \Delta^{(2)}_{\mu,\nu}(\theta_l,\tilde{U}_l,\tilde{V}_l).
    \end{equation}
\end{definition}
One can check that this defines again an unbiased estimator of $\SW^2_2(\mu,\nu)$, which can again be computed in time $O(L(N+M)d)$. One possible choice of $\eta$ is $\dd\eta_k(x,y) = k(k+1) (y-x)^{k-1} \dd x \dd y$ for some $k \in (0,1]$. This ensures that, if $(X,Y) \sim \eta_k$, then $Y-X \sim \mathrm{Beta}(k,2)$ (see \Cref{prop:eta_k_gives_beta}). The lower the value of $k$, the more the law of $\eta$ gives mass to small gaps $y-x$. For $k = 1$, $\eta_1$ is simply the uniform distribution on $T$, so that $\widehat{\SW}_{2,L,\eta_1} = \widehat{\SW}_{2,L}$. We refer to \Cref{appendix:sec:sec:importance_sampling} for a more detailed theoretical and empirical analysis of the conditions under which using this importance sampling strategy with $\eta_k$ leads to a decrease in variance.

\paragraph{Differentiability of the CDF estimators} One of the main limitations of the estimators $\widehat{\SW}_{1,L}$, $\widehat{\SW}_{2,L}$ and $\widehat{\SW}_{2,L,\eta}$ is that they only provide estimates of the SW distance and do not readily give access to the gradient of the Sliced-Wasserstein distance. Indeed, as we saw previously, the evaluations at a fixed point $a \in \R$ of the cumulative distribution functions $F_{\mu_\theta}(a)$ and $F_{\nu_\theta}(a)$ are linear combinations of indicator functions $\bOne_{\sca{x_i}{\theta} \leq a}$, and thus are not differentiable with respect to the points $x_i$ and $y_j$. This can be inconvenient since many practical applications of the Sliced Wasserstein distance require to differentiate it, as in the case for example in generative modeling. The workaround we propose is to \emph{smooth the cumulative distribution functions} to make them differentiable. For any $\rho \in \cP(\R)$ and $a \in \R$, letting $\veps > 0$ be a smoothing parameter, we introduce the smoothed version of the cumulative distribution function of $\rho$
\begin{equation} \label{eq:smoothed_cdf}
    F^{(\veps)}_\rho(a) := \int \sigma(-(x-a-\sqrt{\veps})/\veps) \dd\rho(x),
\end{equation}
where $\sigma$ is the sigmoid function $\sigma(x) = 1/(1+e^{-x})$. We can then introduce smoothed estimators $\widehat{\SW}_{1,L,\veps}$, $\widehat{\SW}_{2,L,\veps}$ and $\widehat{\SW}_{2,L,\eta,\veps}$ which are defined by replacing the CDFs $F_{\mu_\theta}$, $F_{\nu_\theta}$ by their smoothed counterparts $F_{\mu_\theta}^{(\veps)}$, $F_{\nu_\theta}^{(\veps)}$ in respectively \Cref{def:sw1_alt_estimator}, \Cref{def:sw2_alt_estimator} and \Cref{def:sw2_alt_estimator_k}. As the $F^{(\veps)}_\rho$ are only approximations of $F_\rho$, these estimators will be biased. Nevertheless, we can check that for any fixed $\rho \in \cP(\R)$ and $a \in \R$, we have $F^{(\veps)}_\rho(a) \to F_\rho(a)$ when $\veps \to 0^+$, so that the estimators $\widehat{\SW}_{1,L,\veps}$, $\widehat{\SW}_{2,L,\veps}$ and $\widehat{\SW}_{2,L,\eta,\veps}$ will converge respectively to $\widehat{\SW}_{1,L}$, $\widehat{\SW}_{2,L}$ and $\widehat{\SW}_{2,L,\eta}$ when $\veps \to 0^+$. Note that, while the process by which we smooth the CDFs in Equation \eqref{eq:smoothed_cdf} may seem at first glance to be ad hoc, it turns out that these smoothed estimators can be interpreted as (biased) estimators of a Sliced-Wasserstein distance between smoothed versions of $\mu$ and $\nu$: we refer for further details to the theoretical analysis in \Cref{appendix:sec:sec:smoothed_estimator_analysis}.

\section{Practical Applications}\label{sec:applications}
In this section, we review practical applications of the estimators \(\widehat{\SW}_{1,L}\), \(\widehat{\SW}_{2,L}^2\), \(\widehat{\SW}_{2,L,\eta}^2\), and their smoothed variants introduced in \Cref{sec:cdf_estimators}. We refer to these methods as the \emph{CDF-based estimators}, in contrast to the standard estimators \(\widetilde{\SW}_{p,L}^p\), which we call the \emph{quantile-based estimators} since they rely on the quantile representation \eqref{eq:wass_1d_expression_quantile} of the one-dimensional Wasserstein distance.

\paragraph{Dataset parallelization} It is well-known that the quantile-based estimator $\widetilde{\SW}_{p,L}^p$ can be computed in a parallel manner over the number $L$ of directions, since each $\W_p^p(\mu_{\theta_l},\nu_{\theta_l})$ in its definition \eqref{eq:sw_quantile_estimator} can be computed independently. This is still true for the CDF-based estimators, since the terms corresponding to each $(\theta_l,U_l) \in \bS^{d-1} \times [0,1]$ or $(\theta_l,U_l,V_l) \in \bS^{d-1} \times T$ can be computed independently. However, thanks to the properties of cumulative distribution functions, a remarkable feature of the CDF-based $\SW$ estimators is that their computation can additionally be made parallel over $N$, where $N$ is the number of samples of the compared measures.

Consider indeed a setting where we have measures $\mu_1,\mu_2 \in \cP_2(\R^d)$ of the form $\mu_i = \sum_{j=1}^{p_i} \alpha_{i,j} \mu_{i,j}$, where the $\mu_{i,j}$ are probability measures held by $p = p_1 + p_2$ separate devices, and the $\alpha_{i,j}$ are nonnegative weights such that $\sum_{j=1}^{p_i} \alpha_{i,j} = 1$, $i=1,2$. Then, in this case, for every $i=1,2$ and $\theta \in \bS^{d-1}$, the CDF of $\mu_{i,\theta}$ has the simple expression
$F_{\mu_{i,\theta}} = \sum_{j=1}^{p_i} \alpha_{i,j} F_{\mu_{i,j,\theta}}$. This allows the computation of the CDF-based $\SW$ estimators to be done in parallel across the $p$ devices using, for instance in the case of $\widehat{\SW}_{1,L}(\mu_1,\mu_2)$, the following algorithm: (1) A central server samples $(\theta_1,U_1),\ldots,(\theta_L,U_L)$ uniformly from $\bS^{d-1} \times [0,1]$ and sends them to the devices, (2) each device $(i,j)$ computes its local CDFs $(F_{\mu_{i,j,\theta_l}}(\tilde{U}_l))_{l=1}^L$ and sends them to the server, (3) which sums them to calculate the CDFs $(F_{\mu_{i,\theta_l}}(\tilde{U}_l))_{l=1}^L$, $i=1,2$ and uses them  to compute the value of $\widehat{\SW}_{1,L}(\mu_1,\mu_2)$. Of course, this algorithm can be adapted with minor changes to the CDF-based $\SW_2$ estimators as well as their smoothed variants.

\paragraph{Federated learning} Federated learning is a machine learning paradigm in which several clients collaborate to train a global model using locally stored data without sharing it \citep{kairouz2021advancesopenproblemsfederated}. While there exists a number of works aiming to devise dedicated algorithms for computing in a federated setting optimal transport distances such as the Wasserstein distance \citep{rakotomamonjy2024federated} or the Sinkhorn divergence \citep{kulcsar2025federated}, it is notable that the parallel algorithm described in the previous paragraph can be straightforwardly adapted to a federated setting. Indeed, in this case, the $p = p_1 + p_2$ devices that the computation is parallelized over are not computation units within a single machine or cluster, but rather distinct clients with whom the central server communicates through a network, each holding a private dataset $\mu_{i,j}$. It is thus possible to compute the CDF-based $\SW$ estimators in a federated setting. Moreover, while there may be privacy concerns regarding this algorithm (notably regarding the information revealed by the transmission of the local CDFs $F_{\mu_{i,j,\theta_l}}$), it can be given formal privacy guarantees in terms of $(\veps,\delta)$-differencial privacy through the addition of a noising scheme. We refer to \Cref{appendix:sec:fed_sw} for a detailed description of the federated algorithms, pseudocode, and analysis of their privacy.

\paragraph{Gaussian mixtures} Although we focused mainly in \Cref{sec:cdf_estimators} on the case where the compared measures $\mu$ and $\nu$ are empirical measures, it must be noted that the CDF-based estimators defined in \Cref{def:sw1_alt_estimator}, \Cref{def:sw2_alt_estimator} and \Cref{def:sw2_alt_estimator_k} only require that the CDFs $F_{\mu_\theta}$ and $F_{\nu_\theta}$ are tractable. In fact, one strength of these estimators, compared to the usual quantile-based estimators $\widetilde{\SW}_p$, is that they are well-suited to the case where $\mu$ and $\nu$ have a form such that their CDFs $F_\mu$ and $F_\nu$ are more practical to compute than their quantile functions $F_\mu^{-1}$ and $F_\nu^{-1}$. This, for instance, is the case of \emph{Gaussian mixtures}. Recall that a Gaussian mixture is a probability measure $\mu \in \cP_2(\R^d)$ of the form $\mu = \sum_{k=1}^n \alpha_k \cN(m_k, \Sigma_k)$ where $\alpha_1,\ldots,\alpha_n \in [0,1]$ are such that $\sum_{i=1}^n \alpha_i = 1$, $m_1,\ldots,m_n \in \R^d$, and $\Sigma_1,\ldots,\Sigma_n \in S_d^{++}(\R)$ are symmetric positive definite matrices. Then, for every $\theta \in \bS^{d-1}$ and every $x \in \R$, the projected measure $\mu_\theta$ and its CDF $F_{\mu_\theta}(x)$ have the form
\begin{equation}
    \mu_\theta = \sum_{k=1}^n \alpha_k \cN(m_{k,\theta}, \sigma_{k,\theta}^2), \quad F_{\mu_\theta}(x) = \sum_{k=1}^n \alpha_k \Phi\left(\frac{x - m_{k,\theta}}{\sigma_{k,\theta}}\right)
\end{equation}
with $m_{k,\theta} := \sca{\theta}{m_k}$ and $\sigma_{k,\theta}^2 := \theta^T \Sigma_k \theta$ for every $k$, and with $\Phi$ the CDF of the standard normal distribution $\cN(0,1)$. In particular $\mu_\theta$ is also a Gaussian mixture. Since $\Phi$ can be efficiently numerically approximated and differentiated, we can thus directly use this expression in the computation of the CDF-based estimators, which will furthermore be straightforwardly differentiable with respect to the parameters $\alpha_k,m_k,\Sigma_k$ of the Gaussian mixture (unlike in the case of empirical measures, in \Cref{sec:cdf_estimators}, for which we needed to smooth the CDFs to obtain differentiable estimators !). On the other hand, there are no practical expression of the quantile functions $F_{\mu_\theta}^{-1}$ of the projections of a Gaussian mixture. Therefore, in order to compute the usual quantile-based estimator, as is done for example in \citep{kolouri2018slicedgaussian}, one needs to approximate the quantile functions by numerically inverting the CDFs $F_{\mu_\theta}$.

\begin{remark}
    If $\mu$ is a Gaussian mixture, the bounds $m_\theta^\pm$ defined in \Cref{sec:cdf_estimators} are infinite: indeed, the projected measure $\mu_\theta$ has full support on $\R$, so that $m_\theta^\pm = \pm\infty$. This makes it impossible to sample on $[0,1]$ or $T$ and rescale to $m_\theta^- \leq u \leq m_\theta^+$ or $m_\theta^- \leq u \leq v \leq m_\theta^+$. One possible workaround is to redefine $m_\theta^\pm$ as the $\max$ and $\min$ of $S(\mu_\theta) \cup S(\nu_\theta)$, where $S(\mu_\theta) := \spt(\mu_\theta)$ when $\mu_\theta$ has bounded support, and, when $\mu$ is a Gaussian mixture of the form $\mu = \sum_k \alpha_k \cN(m_k,\Sigma_k)$, $S(\mu_\theta) := \bigcup_{k=1}^n [m_{k,\theta} - \tau \sigma_{k,\theta}, m_{k,\theta} + \tau \sigma_{k,\theta}]$ for some fixed $\tau > 0$. In other words, we ignore the points that are not within $\tau$ standard deviations of the mean of at least one of the components of $\mu_\theta$. This slightly biases the estimators downward (as this implies truncating the integrals in \eqref{eq:sw1_alternate} and \eqref{eq:sw2_alternate}), but for $\tau$ large enough the bias should be negligible\footnote{Recall that 99.7\% of the mass of a normal distribution is within 3 standard deviations of its mean. Therefore, if $\tau = 3$, then $\mu_\theta(\widetilde{\spt}(\mu_\theta)) \geq 0.997$ for every $\theta \in \bS^{d-1}$.}.
\end{remark}

\section{Experiments}\label{sec:experiments}

In this section we describe several experiments showcasing applications of the CDF-based $\SW$ estimators. Our code will be made publicly available.

\paragraph{Generative modeling} In order to demonstrate the suitability of the smoothed CDF-based estimators for practical machine learning applications, we train a Generative Adversarial Network (GAN) on the datasets MNIST \citep{lecun2010mnist} and CelebA (rescaled to $64 \times 64$) \citep{liu2015deep}, using the $\SW_2$ distance as the discriminator loss, which we estimate using the quantile-based estimator $\widetilde{\SW}_{2,L}^2$ and the CDF-based estimator $\widehat{\SW}_{2,L,\eta_k,\veps}^2$ with $k = 1$ and $k = 0.05$. Following the rule of thumb outlined in \Cref{appendix:sec:sec:sec:exp_time_vs_acc}, we take $L = 10000$ for the quantile estimator and $L = 100000$ for the CDF estimators to obtain comparable variance, and following the results of \Cref{appendix:sec:sec:sec:tune_smoothing}, we take $\veps = 0.01$. We reuse the architectures of \citep{nguyen2022amortized} (in particular, we use spectral normalization), train for 100 epochs, with a batch size of 128, using Adam with learning rate $0.0002$. We use the Fréchet Inception Distance (FID) \citep{heusel2017gans} to evaluate the quality of the images generated by the trained models. The FID scores obtained at the end of the training runs are reported in \Cref{tab:fid_scores}. We observe that the training runs using the CDF-based estimators achieve lower FID scores than the ones using the usual quantile-based estimator\ak{stylé !}, while it does not seem that lowering the importance sampling parameter $k$ systematically improves or worsens the final FID score. We refer to \Cref{appendix:sec:details_gans} for more detailed information on the architecture of the model and on the training.

\begin{table}[H]
    \centering
    \caption{FID scores (mean and standard deviation over 5 computations) for GANs trained with different datasets and estimators of the SW distance.}
    \begin{tabular}{cccc}
        \hline
        \textbf{Estimator} & \textbf{Quantile} & \textbf{CDF} (k=1) & \textbf{CDF} (k=0.05) \\
         \hline
         \textbf{MNIST} & $5.41 \pm 0.02$ & $4.69 \pm 0.10$ & $5.00 \pm 0.09$ \\
         \hline
         \textbf{CelebA} & $10.08 \pm 0.10$ & $9.47 \pm 0.02$ & $9.34 \pm 0.07$ \\
         \hline
    \end{tabular}
    \label{tab:fid_scores}
\end{table}

\begin{figure}
    \centering
    \includegraphics[width=\linewidth]{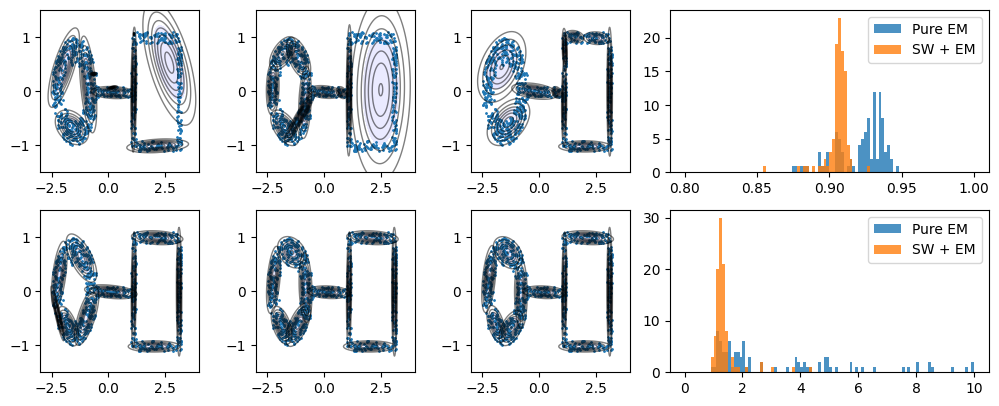}
    \caption{Left: Results of fitting a Gaussian mixture model with 10 modes to a "ring-line-square" for three random initializations of the GMM, using the EM algorithm (top row) or by minimizing the CDF-based SW estimator with a few EM steps at the end (bottom row). Right: Histograms across 100 runs of the negative log likelihood (top) and of the $\SW_2^2$ distance (bottom) between the fitted GMM and the target point cloud for both methods.}
    \label{fig:gmm_sw_vs_em_results}
\end{figure}

\paragraph{Gaussian mixture models} The standard method for fitting the parameters of a Gaussian mixture model (GMM) $\mu = \sum_{k=1}^n \alpha_k \cN(m_k, \Sigma_k)$ to a set of points $X \in (\R^d)^N$ is the expectation-maximization (EM) algorithm, which minimizes the negative log-likelihood (NLL) of $X$ with respect to $\mu$. However, while it guarantees convergence to a stationary point of the NLL, random initializations of the EM algorithm converge to bad local minima with high probability\citep{kolouri2018slicedgaussian}. Another method of fitting the GMM is to minimize the Sliced-Wasserstein distance between $\mu$ and the empirical measure representing the $X$. This method was investigated by \citet{kolouri2018slicedgaussian}, who showed that it is more robust to random initialization and yields solutions closer to the global minimum for both the $\SW$ and NLL objectives than the EM algorithm. While they used the quantile-based estimator of the $\SW$ in their analysis, we show that the benefits of their method can also be attained using the CDF-based estimators. To this end, we reproduce their experiment in \citep[Section 4.1]{kolouri2018slicedgaussian} using the CDF-based estimators instead of the quantile-based ones. We consider a dataset of two-dimensional points $X \in (\R^2)^N$ consisting of a ring, a line and a square, and we fit a GMM with $K = 10$ components to $X$ using one method among: (1) 200 steps of the EM algorithm or (2) 175 gradient descent steps of the loss $\widehat{\SW}_{2,L}^2$ using RMSProp followed by 25 steps of the EM algorithm. 

For the second method, we take $L = 1000$, and we set the RMSProp learning rates to be $0.05$ for the $m_k$ and $0.005$ for the $\Sigma_k$. The component weights $\alpha_k$ are fixed at $\alpha_k = 1/K$ during the RMSProp phase. The CDF-based $\SW$ estimator is calculated as explained in \Cref{sec:applications}, using $\tau = 3$ to compute the bounds $m_\theta^\pm$. The results are shown in \Cref{fig:gmm_sw_vs_em_results}. We see that, just as in \citep{kolouri2018slicedgaussian}, minimizing the $\SW$ leads to a better fit of the model both visually, in terms of the $\SW$ distance, and in terms of the NLL, than the EM algorithm.\ak{reporter les résultats avec le quantile-based estimator en appendix? dire que cest comparable?}

\paragraph{Boosting federated learning algorithms} One of the issues which is known to impair federated learning algorithms is the heterogeneity of the client datasets. To address this, \citet{rakotomamonjy2024federated} proposed an approach consisting of clustering clients with "similar" datasets and training a separate global model on each cluster. This approach makes use of the Optimal Transport Dataset Distance (OTDD) defined by \citet{alvarez2021dataset}, which can be computed in a federated way, thanks to the federated estimator of the $\W_2$ distance (FedWaD) introduced by \citet{rakotomamonjy2024federated}. We propose to use instead of the OTDD a $\SW$-based dataset distance such as the Maximum Mean Discrepancy with Riesz $\SW$ kernel (MMDSW) \citep{bonet2025flowing}, which we can compute in a federated way (see \Cref{appendix:sec:sec:otdd_and_mmdsw}) using our federated algorithms for the CDF-based $\SW$ estimators (see \Cref{sec:applications}). The MMDSW distance indeed has the advantage of being faster to compute than OTDD (both in practice and in terms of theoretical time complexity). Moreover, the federated CDF-based $\SW$ estimator (with which we compute MMDSW) has a number of advantages over FedWaD (with which OTDD is computed): first, the $\SW$ estimator only requires a single exchange of information between clients and server, while FedWaD requires approximating an interpolating measure over several iterations ; moreover, unlike FedWaD, the $\SW$ estimator can easily be given formal privacy guarantees (see \Cref{appendix:sec:fed_sw}).

We reimplemented the experiments of \citet{rakotomamonjy2024federated} using the MMDSW instead of the OTDD. We compute the MMDSW using the $\widehat{\SW}_{2,L}^2$ estimator with $L = 10000$. We train classifiers on the MNIST  \citep{lecun2010mnist} and CIFAR-10 \citep{krizhevsky2009learning} datasets, using the federated algorithms FedAvg, FedPer \citep{arivazhagan2019federated} and FedRep \citep{collins2021exploiting}, and with a number of clients $C \in \{20,40,100\}$. We use the spectral clustering algorithm of \texttt{scikit-learn} \citep{pedregosa2011scikit}, with $K = 5$ clusters. The results are shown in \Cref{tab:federated}. In the ``Structure" columns of the table, we simulate high heterogeneity by imposing a cluster structure on the client datasets: we assign to each client a random pair of labels chosen among $(0,1),(2,3),(4,5),(6,7),(8,9)$, and, for every label $l$, the samples of the dataset with this label are split between the clients to whom the label $l$ has been assigned. Conversely, in the ``No structure" columns, no specific cluster structure has been imposed on the client datasets. We see in particular that our boosting procedure consistently increases the performance of the trained classifiers. More detailed explanations and results can be found in \Cref{appendix:sec:details_boosting_fl}.

\begin{table*}
    \centering
    \caption{Performance for boosted FL algorithms. The average uplifts for OTDD are taken from Table 1 of \citet{rakotomamonjy2024federated} ("Affinity" columns). Note that the higher uplifts for OTDD are in part due to the fact that the baseline performance ("Vanilla") is often lower in \citet{rakotomamonjy2024federated}, so that in our reimplementation the boosting has less room to improve from.}
    
    \resizebox{\linewidth}{!}{
    \begin{tabular}{cccccccccc}
        \hline
         & & \multicolumn{4}{c}{MNIST} & \multicolumn{4}{c}{CIFAR-10} \\ \hline
         & & \multicolumn{2}{c}{Structure} & \multicolumn{2}{c}{No structure} & \multicolumn{2}{c}{Structure} & \multicolumn{2}{c}{No structure} \\ \hline
         & & Vanilla & Clustering & Vanilla & Clustering & Vanilla & Clustering & Vanilla & Clustering \\ \hline
         \multirow{3}{2em}{FedAvg} & 20 & 58.0 $\pm$ 1.7 & \textbf{99.2} $\pm$ \textbf{0.1} & 65.6 $\pm$ 5.0 & \textbf{84.0} $\pm$ \textbf{4.3} & 32.9 $\pm$ 0.9 & \textbf{82.6} $\pm$ \textbf{3.9} & 35.5 $\pm$ 1.4 & \textbf{60.6} $\pm$ \textbf{2.2}\\
         & 40 & 51.7 $\pm$ 3.6 & \textbf{99.0} $\pm$ \textbf{0.0} & 71.3 $\pm$ 3.7 & \textbf{77.0} $\pm$ \textbf{3.3} & 31.8 $\pm$ 1.1 & \textbf{84.0} $\pm$ \textbf{0.3} & 35.2 $\pm$ 1.9 & \textbf{57.5} $\pm$ \textbf{1.8}\\
         & 100 & 56.5 $\pm$ 4.5 & \textbf{98.8} $\pm$ \textbf{0.0} & 70.3 $\pm$ 2.9 & \textbf{72.5} $\pm$ \textbf{2.3} & 27.2 $\pm$ 2.3 & \textbf{82.4} $\pm$ \textbf{0.2} & 31.2 $\pm$ 1.1 & \textbf{47.7} $\pm$ \textbf{2.2}\\
         \multirow{3}{2em}{FedPer} & 20 & 98.0 $\pm$ 0.5 & \textbf{99.2} $\pm$ \textbf{0.0} & 95.7 $\pm$ 1.1 & \textbf{97.9} $\pm$ \textbf{0.3} & 81.6 $\pm$ 0.5 & \textbf{84.7} $\pm$ \textbf{0.2} & 81.9 $\pm$ 2.5 & \textbf{84.2} $\pm$ \textbf{2.4}\\
         & 40 & 97.5 $\pm$ 0.4 & \textbf{99.0} $\pm$ \textbf{0.0} & 94.7 $\pm$ 0.9 & \textbf{96.5} $\pm$ \textbf{0.4} & 80.5 $\pm$ 0.6 & \textbf{84.1} $\pm$ \textbf{0.3} & 81.7 $\pm$ 1.0 & \textbf{82.9} $\pm$ \textbf{1.1}\\
         & 100 & 95.9 $\pm$ 2.4 & \textbf{98.8} $\pm$ \textbf{0.0} & 94.7 $\pm$ 1.2 & \textbf{96.1} $\pm$ \textbf{0.6} & 76.9 $\pm$ 0.8 & \textbf{82.3} $\pm$ \textbf{0.2} & 77.8 $\pm$ 1.1 & \textbf{80.5} $\pm$ \textbf{0.7}\\
         \multirow{3}{2em}{FedRep} & 20 & 98.2 $\pm$ 0.1 & \textbf{98.8} $\pm$ \textbf{0.1} & 97.3 $\pm$ 0.2 & \textbf{97.9} $\pm$ \textbf{0.3} & 78.7 $\pm$ 0.8 & \textbf{82.6} $\pm$ \textbf{0.2} & 80.7 $\pm$ 2.6 & \textbf{83.2} $\pm$ \textbf{2.4}\\
         & 40 & 97.4 $\pm$ 0.2 & \textbf{98.7} $\pm$ \textbf{0.1} & 95.7 $\pm$ 1.2 & \textbf{97.3} $\pm$ \textbf{0.5} & 76.5 $\pm$ 0.6 & \textbf{81.3} $\pm$ \textbf{0.3} & 78.2 $\pm$ 1.2 & \textbf{81.1} $\pm$ \textbf{1.1}\\
         & 100 & 92.9 $\pm$ 4.5 & \textbf{98.4} $\pm$ \textbf{0.0} & 93.8 $\pm$ 0.5 & \textbf{96.1} $\pm$ \textbf{0.5} & 64.2 $\pm$ 1.1 & \textbf{72.9} $\pm$ \textbf{1.0} & 70.4 $\pm$ 0.8 & \textbf{72.9} $\pm$ \textbf{0.7}\\
         \hline 
         \multicolumn{2}{l}{Average uplift} &&&&&&&& \\
         \multicolumn{2}{l}{MMDSW} &  - & 15.7 $\pm$ 19.8 &  - & 6.2 $\pm$ 9.7 &  - & \textbf{24.3} $\pm$ \textbf{27.2} &  - & 15.7 $\pm$ 18.8\\
         \multicolumn{2}{l}{OTDD} & -  & \textbf{26.4} $\pm$ \textbf{27.5} & -  & \textbf{12.7} $\pm$ \textbf{14.6} & -  & 17.6 $\pm$ 19.6 & -  & \textbf{18.8} $\pm$ \textbf{16.6}\\ \hline
    \end{tabular}}
    \label{tab:federated}
\end{table*}

\section{Conclusion}

In this work, we introduced a new class of estimators of the Sliced-Wasserstein distance which leverage the properties of cumulative distribution functions of projections of measures. We discussed applications of these estimators, some novel (such as federated learning), and illustrated them with various practical experiments.

\paragraph{Limitations} The CDF-based estimators suffer from a somewhat worse performance-accuracy trade-off than the quantile-based estimators, that the importance sampling strategy does not fully alleviate (see \Cref{appendix:sec:sec:sec:exp_time_vs_acc}). We lack a complete theoretical understanding of the smoothed CDF-based estimators, such as of their bias or of their convergence properties when the smoothing parameters $\veps$ converges to $0$. While we noted that the federated $\SW$ estimators could be made differentially private through a noising scheme (\Cref{appendix:sec:fed_sw}), we did not investigate the trade-off between privacy and accuracy resulting from the implementation of such a scheme.

\paragraph{Future work} Future avenues for research could include investigating the topological and statistical properties of these estimators in the same vein as \citet{nadjahi2020,bayraktar2021strong,goldfeld2021sliced}, their robustness and computational guarantees (as in \citet{nietert2022statistical}), implementing more sophisticated variance minimization strategies to improve the performance-accuracy trade-off, or adapting them to other settings such as measures supported on manifolds, slicing unbalanced optimal transport \citep{bonet2025slicing} or data with a latent hierarchy \citep{lin2025tree}.

\bibliographystyle{plainnat}
\bibliography{biblio}

%%%%%%%%%%%%%%%%%%%%%%%%%%%%%%%%%%%%%%%%%%%%%%%%%%%%%%%%%%%%

\appendix
\onecolumn

\section{Supplementary results on the CDF estimators}

\subsection{Importance sampling} \label{appendix:sec:sec:importance_sampling}

Recall that for every $\alpha, \beta > 0$, the beta distribution $\Beta(\alpha, \beta)$ is the probability distribution on the segment $[0,1]$ whose probability density function is given by
\begin{equation}
    \Beta(x|\alpha,\beta) := \frac{x^{\alpha-1} (1-x)^{\beta-1}}{B(\alpha,\beta)}, \quad x \in (0,1)
\end{equation}
where
\begin{equation}
    B(\alpha,\beta) := \frac{\Gamma(\alpha)\Gamma(\beta)}{\Gamma(\alpha+\beta)}
\end{equation}
and $\Gamma$ is the Gamma function.

\begin{proposition} \label{prop:eta_k_gives_beta}
    Let $k \in (0,1)$ and let $X,Y$ be real random variables of law $(X,Y) \sim \eta_k$, where $\eta_k$ is the probability distribution defined in \Cref{sec:cdf_estimators}. Then $Y - X \sim \Beta(k,2)$. 
\end{proposition}

\begin{proof}
    Since we have 
    \begin{equation}
        B(k,2) = \frac{\Gamma(k)\Gamma(2)}{\Gamma(k+2)} = \frac{(k-1)!}{(k+1)!} = \frac{1}{k(k+1)}
    \end{equation}
    (recall that $\Gamma(n) = (n-1)!$ for every integer $n > 0$), we have for every $x \in (0,1)$ that $\Beta(x|k,2) = k(k+1) x^{k-1} (1-x)$. Now, let $\varphi : [0,1] \to \R_+$ be any Borel measurable function, then
    \begin{align}
        \E[\varphi(Y-X)] &= \int_T \varphi(y-x) \dd\eta_k(x,y) = \int_0^1 \int_x^1 \varphi(y-x) k(k+1)(y-x)^{k-1} \dd y \dd x \\
        &= \int_0^1 \int_0^{1-x} \varphi(u) k(k+1) u^{k-1} \dd u \dd x \\
        &= \int_0^1 \int_0^{1-u} \varphi(u) k(k+1) u^{k-1} \dd x \dd u \\
        &= \int_0^1 \varphi(u) k(k+1) u^{k-1} (1-u) \dd u = \int_0^1 \varphi(u) \Beta(u|k,2) \dd u \\
        &= \E_{U \sim \Beta(k,2)}[\varphi(U)].
    \end{align}
    where we made the change of variables $u = y-x$ in the second line. Therefore, we indeed have $Y-X \sim \Beta(k,2)$.
\end{proof}

This proposition shows that the distributions $\eta_k$ are good candidates to be used with the $\widehat{\SW}_{2,L,\eta}^2$. Indeed, we expect this estimator to have decreased variance (compared to the estimator $\widehat{\SW}_{2,L}^2$) when the distribution $\eta$, which it uses to sample from the triangle $T = \{0 \leq x \leq y \leq 1\}$, gives more weight to pairs $(x,y) \in T$ where the gap $y-x$ is small. This requirement is satisfied by the $\eta_k$ when $0 < k < 1$: indeed, by the previous proposition, if $(X,Y) \sim \eta_k$, then $Y-X$ follows the law $\Beta(k,2)$, which has density $\propto x^{k-1}(1-x)$ on $[0,1]$. Since $x^{k-1} \to +\infty$ when $x \to 0^+$, this law is concentrated near 0, and in fact the closer $k$ is to 0, the more it is concentrated.

We now provide a condition on the measures $\mu,\nu$ under which using the distributions $\eta_k$ does lead to a decrease in variance of the CDF-based estimator. We first have the following lemma:

\begin{lemma} \label{lemma:importance_sampling_lower_variance}
    Let $L > 0$, $\veps \in (0,e^{-1/2})$, and $\mu, \nu \in \cP(\R^d)$ be compactly supported such that for every $\theta \in \bS^{d-1}$ and $(u,v) \in T$ with $v-u > \veps$, one has $\Delta^{(2)}_{\mu,\nu}(\theta,m_\theta^-+r_\theta u,m_\theta^- + r_\theta v) = 0$ (recall the notations $\Delta^{(2)}_{\mu,\nu}$, $m_\theta^\pm$ and $r_\theta$ from \Cref{sec:cdf_estimators}). Then there exists $k_0 = k_0(\veps) < 1$ such that the function $k \in (0,1] \mapsto \mathrm{Var}(\widehat{\SW}_{2,L,\eta_k}^2(\mu,\nu))$ is nondecreasing on $[k_0,1]$.
\end{lemma}

\begin{proof}
    Let $k \in (0,1]$. Recall that $\dd\eta_k(u,v) = f_k(u,v) \dd u \dd v$ with $f_k(u,v) = k(k+1)(v-u)^{k-1}$. Recall from \eqref{eq:sw2_alt_estimator_2} that
    \begin{equation}
        \widehat{\SW}^2_{2,L,\eta_k} = \frac{1}{L} \sum_{l=1}^L \frac{2r_{\theta_l}^2}{f_k(U_l,V_l)} \Delta^{(2)}_{\mu,\nu}(\theta_l,\tilde{U}_l,\tilde{V}_l)
    \end{equation}
    with $(\theta_l,U_l,V_l) \sim \cU(\bS^{d-1}) \times \eta_k$ i.i.d., with the notation $\tilde{U}_l = \Rescale(U_l,m_{\theta_l}^-,m_{\theta_l}^+)$ and similarly for $\tilde{V}_l$. Therefore, we have
    \begin{align}
        \mathrm{Var}\left(\widehat{\SW}^2_{2,L,\eta_k}(\mu,\nu)\right) &= \frac{1}{L} \mathrm{Var}\left(\frac{2r_\theta^2}{f_k(u,v)} \Delta^{(2)}_{\mu,\nu}(\theta,\tilde{U},\tilde{V})\right) \\
        &= \frac 1L \left(\E\left[\left(\frac{2r_\theta^2}{f_k(u,v)} \Delta^{(2)}_{\mu,\nu}(\theta,\tilde{U},\tilde{V})\right)^2\right] - \SW_2^4(\mu,\nu)\right)
    \end{align}
    where $(\theta,U,V) \sim \cU(\bS^{d-1}) \times \eta_k$ and $\tilde{U} = \Rescale(U,m_\theta^-,m_\theta^+)$ and similarly for $\tilde{V}$, and we used the unbiasedness of $\widehat{\SW}^2_{2,L,\eta_k}$ in the second line. Using the fact that $f_k/2$ is the density of $\eta_k$ with respect to $\cU(T)$ (as $\dd\cU(T)(u,v) = 2\dd u \dd v$), we have in fact
    \begin{equation}
        \mathrm{Var}\left(\widehat{\SW}^2_{2,L,\eta_k}(\mu,\nu)\right) = \frac 1L \left(\E\left[\frac{2r_\theta^4}{f_k(u,v)} \Delta^{(2)}_{\mu,\nu}(\theta,\tilde{U},\tilde{V})^2\right] - \SW_2^4(\mu,\nu)\right)
    \end{equation}
    where this time $(\theta,U,V) \sim \cU(\bS^{d-1} \times T)$. If we fix
    \begin{equation}
        \Phi(k) := \E_{(\theta,U,V) \sim \cU(\bS^{d-1} \times T)} \left[\frac{2r_\theta^4}{f_k(u,v)} \Delta^{(2)}_{\mu,\nu}(\theta,\tilde{U},\tilde{V})^2\right],
    \end{equation}
    then we have $\mathrm{Var}\left(\widehat{\SW}^2_{2,L,\eta_k}(\mu,\nu)\right) = \frac 1L (\Phi(k) - \SW_2^4(\mu,\nu))$ so the dependency in $k$ of the variance of the importance sampling estimator is determined by $\Phi$. Now, we have
    \begin{align}
        \Phi(k) &= \int_{\bS^{d-1}} \int_0^1 \int_u^1 \frac{2r_\theta^4}{f_k(u,v)} \Delta_{\mu,\nu}^{(2)}(\theta, \tilde{u}, \tilde{v})^2 2 \dd v \dd u \dd \theta \\
        &= \frac{1}{k(k+1)} \int_{\bS^{d-1}} \int_0^1 \int_u^1 \frac{4r_\theta^4}{(v-u)^{k-1}} \Delta_{\mu,\nu}^{(2)}(\theta, \tilde{u}, \tilde{v})^2 \dd v \dd u \dd \theta
    \end{align}
    where $\tilde{u}$ is shorthand for $\Rescale(u,m_\theta^-,m_\theta^+)$. Thus, taking the change of variable $r = v-u$, we obtain
    \begin{align}
        \Phi(k) &= \frac{1}{k(k+1)} \int_{\bS^{d-1}} \int_0^1 \int_0^{1-u} \frac{4r_\theta^4}{r^{k-1}} \Delta_{\mu,\nu}^{(2)}(\theta, \tilde{u}, \widetilde{u+r})^2 \dd r \dd u \dd \theta \\
        &= \frac{1}{k(k+1)} \int_0^1 \frac{1}{r^{k-1}} \int_{\bS^{d-1}} \int_0^{1-r} 4r_\theta^4 \Delta_{\mu,\nu}^{(2)}(\theta, \tilde{u}, \widetilde{u+r})^2 \dd u \dd \theta \dd r \\
        &= \frac{1}{k(k+1)} \int_0^1 \frac{1}{r^{k-1}} \Delta(r) \dd r
    \end{align}
    where $\Delta : [0,1] \to \R_+$ is the function defined by
    \begin{equation}
        \Delta(r) := \int_{\bS^{d-1}} \int_0^{1-r} 4r_\theta^4 \Delta_{\mu,\nu}^{(2)}(\theta, \tilde{u}, \widetilde{u+r})^2 \dd u \dd \theta.
    \end{equation}
    Since $\Delta_{\mu,\nu}^{(2)}(\theta,u,v)$ is nondecreasing in $u$ and nonincreasing in $v$, it is clear that $\Delta$ is nonincreasing. Moreover, the assumption of the proposition implies that $\Delta(r) = 0$ whenever $r > \veps$. Differentiating $\Phi$ with respect to $k$, we obtain
    \begin{equation}
        \Phi'(k) = -\frac{1}{k(k+1)} \int_0^1 \left(\frac{1}{k} + \frac{1}{k+1} + \ln(r)\right) \frac{\Delta(r)}{r^{k-1}} \dd r.
    \end{equation}
    Noticing that the function $g_k(r) = \frac{1}{k} + \frac{1}{k+1} + \ln(r)$ has on $[0,1]$ the antiderivative $G_k$ defined by
    \begin{equation}
        G_k(r) = \left(\frac{1}{k} + \frac{1}{k+1}\right) \frac{r^{2-k}}{2-k} + \frac{r^{2-k}}{(2-k)^2} ((2-k)\ln(r) - 1),
    \end{equation}
    after an integration by parts\footnote{Since $\Delta$ is nonincreasing, it has bounded variation, so the integration by parts makes sense if we take $\Delta'$ to be its distributional derivative, which is a negative measure.}, since $G_k(0) = 0$ and $\Delta(1) = 0$ (since $\veps < e^{-1/2} < 1$), we obtain
    \begin{equation}
        \Phi'(k) = \frac{1}{k(k+1)} \int_0^1 G_k(r) \Delta'(r) \dd r.
    \end{equation}
    Now, it is not difficult to see that on the interval $(0,1)$, $G_k$ is negative on $(0,r_k)$, zero at $r_k$, and positive on $(r_k,1)$, where $r_k$ is defined by
    \begin{equation}
        r_k := \exp\left(\frac{1}{2-k} - \frac{1}{k} - \frac{1}{k+1}\right)
    \end{equation}
    The function $k \in (0,1] \to r_k$ is continuous and increasing, with $r_1 = e^{-1/2}$ and $r_k \to 0$ when $k \to 0$. Since $\veps < e^{-1/2}$, there exists $k_0 \in (0,1]$ such that $r_{k_0} = \veps$. Thus, for every $k \in (k_0,1]$, we have $r_k > r_{k_0} = \veps$. Since $\Delta$ is nonincreasing, and $\Delta(r) = 0$ for every $r \geq \veps$, this implies that $\Delta'$ is a negative measure supported on $[0,\veps] \subseteq [0,r_k]$, so that
    \begin{equation}
        \Phi'(k) = \frac{1}{k(k+1)} \int_0^1 G_k(r) \Delta'(r) \dd r = \frac{1}{k(k+1)} \int_0^{r_k} G_k(r) \Delta'(r) \dd r \geq 0
    \end{equation}
    since $G_k < 0$ on $(0,r_k)$. Thus $\Phi$ is nondecreasing on $[k_0,1]$, where $k_0$ only depends on $\veps$, and this finishes the proof.
\end{proof}

\begin{remark}
    The proof of \Cref{lemma:importance_sampling_lower_variance} works for arbitrary $m_\theta^\pm$ as long as $\sup(\mu_\theta) \cup \sup(\nu_\theta) \subseteq [m_\theta^-, m_\theta^+]$ for every $\theta \in \bS^{d-1}$, since this is the condition for $\widehat{\SW}^2_{2,L,\eta}(\mu,\nu)$ to be unbiased.
\end{remark}

If $\mu, \nu \in \cP(\R^d)$ are two compactly supported measures, then the $\W_\infty$ distance between them is defined as
\begin{align}
    \W_\infty(\mu,\nu) &:= \inf_{\gamma \in \Pi(\mu,\nu)} \mathrm{esssup}_{(x,y) \sim \gamma} |x-y| \\
    &:= \lim_{p \to \infty} \W_p(\mu,\nu)
\end{align}
(these definitions are equivalent). It is a distance on the set of compactly supported measures, whose induced topology is stronger than the topology of weak convergence (we refer to \citep[Section 2]{bobkovonedimensional2019} for more details). If $\mu,\nu \in \cP(\R^d)$ and $\theta \in \bS^{d-1}$, since $(P_\theta,P_\theta)_\#\gamma \in \Pi(\mu_\theta,\nu_\theta)$ for any $\gamma \in \Pi(\mu,\nu)$ and $|\sca{x}{\theta}-\sca{y}{\theta}| \leq |x-y|$ for every $x,y \in \R^d$, we have $\W_\infty(\mu_\theta,\nu_\theta) \leq \W_\infty(\mu,\nu)$. Then we have the following lemma:

\begin{lemma} \label{lemma:w_infty_controls_imp_sampling_condition}
    Let $\mu, \nu \in \cP(\R^d)$ be compactly supported. Then, for every $\theta \in \bS^{d-1}$ and $u,v \in \R$ such that $v - u \geq \W_\infty(\mu,\nu)$, we have $\Delta^{(2)}_{\mu,\nu}(\theta,u,v) = 0$.
\end{lemma}

\begin{proof}
    Let $\theta \in \bS^{d-1}$ and $u,v \in \R$ such that $v \geq u$ and $\Delta^{(2)}_{\mu,\nu}(\theta,u,v) > 0$. Then, without loss of generality, we may assume that $F_{\mu_\theta}(u) > F_{\nu_\theta}(v)$. However, by \citep[Theorem 2.12]{bobkovonedimensional2019}, we have
    \begin{equation}
        \W_\infty(\mu_\theta,\nu_\theta) = \inf \{h \geq 0 \setcond \forall x \in \R, F_{\nu_\theta}(x-h) \leq F_{\mu_\theta}(x) \leq F_{\nu_\theta}(x+h) \}.
    \end{equation}
    This implies that $v-u < \W_\infty(\mu_\theta,\nu_\theta) \leq \W_\infty(\mu,\nu)$. This finishes the proof.
\end{proof}

\begin{proposition} \label{prop:importance_sampling_lower_variance}
    Let $R > 0$, $L > 0$ and $\mu, \nu \in \cP(B(0,R))$ be such that $\W_\infty(\mu,\nu) < 2Re^{-1/2}$. Assume that, in the estimator $\widehat{\SW}^2_{2,L,\eta}$, the values $m_\theta^\pm$ are set to $\pm R$ for every $\theta \in \bS^{d-1}$. Then there exists $k_0 = k_0(R,\W_\infty(\mu,\nu)) < 1$ such that the function $k \in (0,1] \mapsto \mathrm{Var}(\widehat{\SW}_{2,L,\eta_k}^2(\mu,\nu))$ is nondecreasing on $[k_0,1]$.
\end{proposition}

\begin{proof}
    Let $\veps := \W_\infty(\mu,\nu)/(2R) < e^{-1/2}$. We only need to show that $\mu$ and $\nu$ satisfy the condition of \Cref{lemma:importance_sampling_lower_variance} with $\veps$. Let $u,v \in T$ such that $v-u > \veps$. Let $\theta \in \bS^{d-1}$ and set $\tilde{u} := m_\theta^- + r_\theta u$ and $\tilde{v} := m_\theta^- + r_\theta v$. Then $r_\theta = m_\theta^+ - m_\theta^- = 2R$, so that $\tilde{v} - \tilde{u} = r_\theta (v-u) = 2R (v-u) > 2R \veps = \W_\infty(\mu,\nu)$, and by \Cref{lemma:w_infty_controls_imp_sampling_condition} we have $\Delta^{(2)}_{\mu,\nu}(\theta,\tilde{u},\tilde{v}) = 0$. Thus, we can apply \Cref{lemma:importance_sampling_lower_variance}, and this finishes the proof. 
\end{proof}

Thus, \Cref{prop:importance_sampling_lower_variance} asserts that, if the measures $\mu,\nu$ are close enough in terms of the $\W_\infty$ distance, there exists a $k < 1$ such that performing importance sampling with the distribution $\eta_k$ leads to a decrease of the variance of the CDF-based Sliced-Wasserstein estimator.

However, this results does not tell us the concrete value of $k$ that we should choose in practice. To guide our choice, we perform the following simple experiment. We sample two clouds of $N = 1000$ points in the plane from the standard normal distribution $\cN(0,I_2)$ and denote $\mu$ and $\nu$ their corresponding empirical measures. We then compute $10^4$ times the estimator $\widehat{\SW}_{2,L,\eta_k}^2$ with $L = 1000$ and for various values of $k$. The resulting histograms are shown in \Cref{fig:importance_sampling_k_choice}. Remembering that the estimator $\widehat{\SW}^2_{2,L,\eta_k}$ for $k = 1$ is simply the estimator $\widehat{\SW}^2_{2,L}$ without importance sampling (as $\eta_1$ is the uniform measure on $T$), we observe that we indeed achieve a lower dispersion of the values when $k < 1$, with in this case an optimum near $k \approx 0.1$. On the other hand, very low values of $k$, of the order of $k \approx 0.01$, lead to an increase of the dispersion of the estimator.

\begin{figure}
    \centering
    \includegraphics[width=\textwidth]{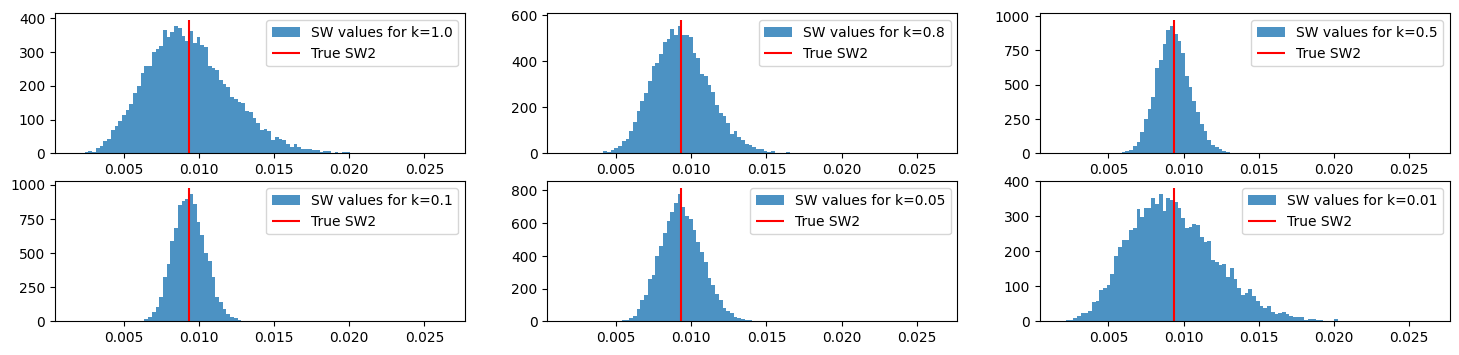}
    \caption{Distribution of the values of $\widehat{\SW}_{2,L,\eta_k}^2$ for various choices of $k$. The red lines denotes the mean of all the computations of the estimators, which is expected to be close to the true $\SW_2^2$ value.}
    \label{fig:importance_sampling_k_choice}
\end{figure}

\subsection{The smoothed estimators: a theoretical analysis} \label{appendix:sec:sec:smoothed_estimator_analysis}

Let $\veps > 0$ be a fixed smoothing parameter. Then for every $\rho \in \cP(\R)$, we recall that we defined the smoothed version of the cumulative distribution function of $\rho$ by
\begin{equation} \label{eq:app:smoothed_cdf}
    F^{(\veps)}_\rho(a) := \int \sigma(-(x-a-\sqrt{\veps})/\veps) \dd\rho(x),\quad a \in \R
\end{equation}
where $\sigma$ is the sigmoid function $\sigma(x) = 1/(1+e^{-x})$. For every $L > 0$ and $\eta \in \cP(T)$ with positive density $f$, we formally define the smoothed versions $\widehat{\SW}_{1,L,\veps}$, $\widehat{\SW}_{2,L,\veps}$ and $\widehat{\SW}_{2,L,\eta,\veps}$ of the estimators $\widehat{\SW}_{1,L}$, $\widehat{\SW}_{2,L}$ and $\widehat{\SW}_{2,L,\eta}$ by:
\begin{equation}
    \widehat{\SW}_{1,L,\veps}(\mu,\nu) := \frac 1L \sum_{l=1}^L r_{\theta_l} \Delta^{(1)}_{\mu,\nu,\veps}(\theta_l,\tilde{U}_l)
\end{equation}
with the same notations as in \Cref{def:sw1_alt_estimator} and with $(\theta_1,U_1),\ldots,(\theta_L,U_L)$ i.i.d. with $(\theta_1,U_1) \sim \cU(\bS^{d-1} \times [0,1])$,
\begin{equation}
    \widehat{\SW}_{2,L,\veps}^2(\mu,\nu) := \frac 1L \sum_{l=1}^L r^2_{\theta_l} \Delta^{(2)}_{\mu,\nu,\veps}(\theta_l,\tilde{U}_l,\tilde{V}_l)
\end{equation}
with the same notations as in \Cref{def:sw2_alt_estimator} and with $(\theta_1,U_1,V_1),\ldots,(\theta_L,U_L,V_L)$ i.i.d. with $(\theta_1,U_1,V_1) \sim \cU(\bS^{d-1} \times T)$, and
\begin{equation}
    \widehat{\SW}_{2,L,\eta,\veps}^2(\mu,\nu) := \frac 1L \sum_{l=1}^L \frac{2r^2_{\theta_l}}{f(U_l,V_l)} \Delta^{(2)}_{\mu,\nu,\veps}(\theta_l,\tilde{U}_l,\tilde{V}_l)
\end{equation}
with the same notations as in \Cref{def:sw2_alt_estimator_k} and with $(\theta_1,U_1,V_1),\ldots,(\theta_L,U_L,V_L)$ i.i.d. with $(\theta_1,U_1,V_1) \sim \cU(\bS^{d-1}) \times \eta$, and where
\begin{align}
    \Delta_{\mu,\nu,\veps}^{(1)}(\theta,u) &:= |F^{(\veps)}_{\mu_\theta}(u) - F^{(\veps)}_{\nu_\theta}(u)| \\
    \Delta_{\mu,\nu,\veps}^{(2)}(\theta,u,v) &:= (F^{(\veps)}_{\mu_\theta}(u) - F^{(\veps)}_{\nu_\theta}(v))^+ + (F^{(\veps)}_{\nu_\theta}(u) - F^{(\veps)}_{\mu_\theta}(v))^+.
\end{align}

It is straightforward to check that these estimators have expectancy
\begin{equation}
    \E[\widehat{\SW}_{1,L,\veps}(\mu,\nu)] = \int_{\bS^{d-1}} \int_{m_\theta^-}^{m_\theta^+} |F_{\mu_\theta}^{(\veps)}(u) - F_{\nu_\theta}^{(\veps)}(u)| \dd u \dd \theta
\end{equation}
and
\begin{align}
    \E[\widehat{\SW}_{2,L,\veps}^2(\mu,\nu)] = \E[\widehat{\SW}_{2,L,\eta,\veps}^2(\mu,\nu)] = 2\int_{\bS^{d-1}} & \iint_{m_\theta^- \leq u \leq v \leq m_\theta^+} (F^{(\veps)}_{\mu_\theta}(u) - F^{(\veps)}_{\nu_\theta}(v))^+ \\
    & \quad + (F^{(\veps)}_{\nu_\theta}(u) - F^{(\veps)}_{\mu_\theta}(v))^+ \dd u \dd v \dd\theta.
\end{align}

Therefore, the smoothed estimators can thus be seen as (biased) estimators of the quantities
\begin{equation}
    \SW_{1,\veps}(\mu,\nu) := \int_{\bS^{d-1}} \int_{-\infty}^{+\infty} |F_{\mu_\theta}^{(\veps)}(u) - F_{\nu_\theta}^{(\veps)}(u)| \dd u \dd \theta
\end{equation}
(for $\widehat{\SW}_{1,L,\veps}$) and
\begin{equation}
    \SW_{2,\veps}^2(\mu,\nu) := 2\int_{\bS^{d-1}} \iint_{u \leq v} (F^{(\veps)}_{\mu_\theta}(u) - F^{(\veps)}_{\nu_\theta}(v))^+ + (F^{(\veps)}_{\nu_\theta}(u) - F^{(\veps)}_{\mu_\theta}(v))^+ \dd u \dd v \dd\theta.
\end{equation}
(for $\widehat{\SW}_{2,L,\veps}$ and $\widehat{\SW}_{2,L,\eta,\veps}$). These quantities are in fact equal to a Sliced-Wasserstein distance between smoothed versions of the measures $\mu$ and $\nu$:

\begin{proposition} \label{prop:smoothed_estimators_estimate_sw_of_smoothed}
    There exists a radial probability distribution $\rho \in \cP(\R^d)$ such that
    \begin{align}
        \SW_{1,\veps}(\mu,\nu) &= \SW_1(\mu \ast \rho_\veps, \nu \ast \rho_\veps) \\
        \SW_{2,\veps}(\mu,\nu) &= \SW_2(\mu \ast \rho_\veps, \nu \ast \rho_\veps)
    \end{align}
    where $\rho_\veps := m_{\veps\#}\rho$ with $m_\veps(x) = \veps x$, and $\ast$ denotes the convolution operation between probability measures.
\end{proposition}

\begin{proof}
    The sigmoid function $\sigma$ is known to be the cumulative distribution function of the standard logistic distribution, which we denote $\rho_0$. We will construct a radial probability distribution $\rho \in \cP(\R^d)$ such that $\rho_\theta = \rho_0$ for every $\theta \in \bS^{d-1}$. \newline
    Let $E_1,\ldots,E_k,\ldots$ be a sequence of independent random variables with $E_k \sim \mathrm{Exp}(k^2/2)$ for every $k > 0$, where $\mathrm{Exp}(\lambda)$ denotes the exponential probability distribution with parameter $\lambda$, and let $Z \sim \cN(0,I_d)$ be independent from the $E_k$. Set $S := \sum_{k=1}^\infty E_k$ (which converges almost surely as $\E[E_k] = 2/k^2$), and set $X := \sqrt{S}Z$. Denote by $\rho$ the law of $X$, which is clearly radial. Let $\theta \in \bS^{d-1}$, we want to prove that $\rho_0 = \rho_\theta$, which is equivalent to proving that $\sca{X}{\theta}$ has law $\rho_0$. For this, all we need to do is to prove that the characteristic function of $\sca{X}{\theta}$ is equal to the characteristic function of $\rho_0$, which is $t \in \R \mapsto \pi t/\sinh(\pi t)$. For every $t \in \R$, we have
    \begin{align}
        \E[e^{it\sca{\theta}{X}}] &= \E[e^{it\sqrt{S}\sca{\theta}{Z}}] \\
        &= \E[\E[e^{it\sqrt{S}\sca{\theta}{Z}}|S]] \\
        &= \E[e^{-\frac{1}{2}t^2S}] \\
        &= \prod_{k=1}^\infty \E[e^{-\frac{1}{2}t^2 E_k}] \label{eq:app:l_98}
    \end{align}
    where the $\E[\cdot|S]$ in the second line denotes the conditional expectancy with respect to $S$, the third line is obtained using the fact that the characteristic function of $\cN(\mu,\Sigma)$ is $\xi \in \R^d \mapsto \exp(\mu^T\xi + \frac 12 \xi^T\Sigma\xi)$, and the inversion between the expectancy and the infinite product in the fourth line is justified by the independence of the $E_k$ and standard arguments. Moreover, for every $E \sim \mathrm{Exp}(\lambda)$ and $u \in \R$, it holds
    \begin{equation}
        \E[e^{-uE}] = \int_0^\infty e^{-ux}\lambda e^{-\lambda x} \dd x = \int_0^\infty \lambda e^{-(u+\lambda)x} \dd x = \frac{\lambda}{\lambda + u}.
    \end{equation}
    Injecting this into \eqref{eq:app:l_98}, we deduce that 
    \begin{equation}
        \E[e^{it\sca{\theta}{X}}] = \prod_{k=1}^\infty \E[e^{-\frac{1}{2}t^2 E_k}] = \prod_{k=1}^\infty \frac{k^2}{t^2 + k^2}.
    \end{equation}
    Using the well-known fact that the hyperbolic sine has the infinite product expansion
    \begin{equation}
        \sinh(x) = x \prod_{k=1}^\infty \left(1 + \frac{x^2}{k^2\pi^2}\right)
    \end{equation}
    for every $x \in \R$, we deduce that
    \begin{equation}
        \E[e^{it\sca{\theta}{X}}] = \frac{\pi t}{\sinh(\pi t)}
    \end{equation}
    which is the characteristic function of the standard logistic distribution $\rho_0$. This proves that $\rho_\theta = \rho_0$. \newline
    Now, let $\rho_\veps := m_{\veps\#}\rho$ and $\rho_{0\veps} := m_{\veps\#}\rho_0$, where $m_\veps$ is the multiplication by $\veps$, and denote $\mu_\veps := \mu \ast \rho_\veps$ and $\nu_\veps := \nu \ast \rho_\veps$. Fix $\theta \in \bS^{d-1}$, then we have 
    \begin{equation}
        (\rho_\veps)_\theta = P_{\theta\#}m_{\veps\#}\rho = m_{\veps\#}P_{\theta\#}\rho = m_{\veps\#}\rho_0 = \rho_{0\veps}
    \end{equation}
    and we have $(\mu_\veps)_\theta = \mu_\theta \ast (\rho_\veps)_\theta = \mu_\theta \ast \rho_{0\veps}$ and similarly $(\nu_\veps)_\theta = \nu_\theta \ast \rho_{0\veps}$\footnote{Indeed for any pair of measures $\mu,\nu \in \cP(\R^d)$ it holds $(\mu \ast \nu)_\theta = \mu_\theta \ast \nu_\theta$, since for every measurable $f : \R \to [0,\infty]$ we have
    \begin{equation*}
        \int f \dd(\mu\ast\nu)_\theta = \int f(\sca{x}{\theta}) \dd(\mu\ast\nu)(x) = \iint f(\sca{x+y}{\theta})\dd\mu(x)\dd\nu(y) = \iint f(t+s)\dd\mu_\theta(t)\dd\nu_\theta(s) = \iint f \dd(\mu_\theta\ast\nu_\theta).
    \end{equation*}}. Since the CDF of $\rho_0$ is the sigmoid $F_{\rho_0} = \sigma$, it ensues that the CDF of $\rho_{0\veps}$ is the rescaled sigmoid $F_{\rho_{0\veps}} = \sigma(\cdot/\veps)$. Therefore, for every $a \in \R$, it holds
    \begin{align}
        F_{(\mu_\veps)_\theta}(a) &= F_{\mu_\theta \ast \rho_{0\veps}}(a) \\
        &= \iint \bOne_{t+s \leq a} \dd\rho_{0\veps}(s) \dd\mu_\theta(t) \\
        &= \int F_{\rho_{0\veps}}(a-t) \dd\mu_\theta(t) \\
        &= \int \sigma((a-t)/\veps) \dd\mu_\theta(t) \\
        &= \int \sigma(-(t-(a-\sqrt{\veps})-\sqrt{\veps})/\veps) \dd\mu_\theta(t) \\
        &= F_{\mu_\theta}^{(\veps)}(a - \sqrt{\veps}) \label{eq:app:l_131}
    \end{align}
    and similarly
    \begin{equation} \label{eq:app:l_134}
        F_{(\nu_\veps)_\theta}(a) = F_{\nu_\theta}^{(\veps)}(a - \sqrt{\veps}).
    \end{equation}
    Therefore, the CDFs of the projections of $\mu_\veps$ and $\nu_\veps$ are (up to a translation by $\sqrt{\veps}$) the smoothed CDFs of the projections of $\mu$ and $\nu$. As a consequence, we have
    \begin{align}
        \SW_{1,\veps}(\mu,\nu) &= \int_{\bS^{d-1}} \int_{-\infty}^{+\infty} |F_{\mu_\theta}^{(\veps)}(u) - F_{\nu_\theta}^{(\veps)}(u)| \dd u \dd \theta \\
        &= \int_{\bS^{d-1}} \int_{-\infty}^{+\infty} |F_{\mu_\theta}^{(\veps)}(u-\sqrt{\veps}) - F_{\nu_\theta}^{(\veps)}(u-\sqrt{\veps})| \dd u \dd \theta \label{eq:app:l_140} \\
        &= \int_{\bS^{d-1}} \int_{-\infty}^{+\infty} |F_{(\mu_\veps)_\theta}(u) - F_{(\nu_\veps)_\theta}(u)| \dd u \dd \theta \\
        &= \SW_1(\mu_\veps, \nu_\veps)
    \end{align}
    where we used \eqref{eq:app:l_131} and \eqref{eq:app:l_134} to obtain the third line, and \Cref{th:bobkov_ledoux_sw1} to obtain the fourth line ; and similarly, we have
    \begin{equation}
        \SW_{2,\veps}(\mu,\nu) = \SW_2(\mu_\veps, \nu_\veps)
    \end{equation}
    using this time \Cref{th:bobkov_ledoux_sw2}. This finishes the proof.
\end{proof}

\begin{remark}
    The $-\sqrt{\veps}$ term in the definition \eqref{eq:app:smoothed_cdf} of the smoothed CDFs $F_\rho^{(\veps)}$ serves to ensure the pointwise convergence $F_\rho^{(\veps)}(a) \to F_\rho(a)$ when $\veps \to 0^+$ for every $a \in \R$. Without this term, the convergence would only hold in the case where $a$ is not an atom of $\rho$, since we would instead have $F_\rho^{(\veps)}(a) \to \rho((-\infty,a)) + \frac 12 \rho(\{a\})$ for every $a \in \R$. However, it would not have been consequential to omit this $-\sqrt{\veps}$ term: indeed, the set of atoms of any probability measure is at most countable, so the estimators $\widehat{\SW}_{1,L,\veps}$, $\widehat{\SW}_{2,L,\veps}$ and $\widehat{\SW}_{2,L,\eta,\veps}$ would in this case still converge almost surely to  respectively $\widehat{\SW}_{1,L}$, $\widehat{\SW}_{2,L}$ and $\widehat{\SW}_{2,L,\eta}$ when $\veps \to 0^+$, as the $\tilde{U}_l$ and $\tilde{V}_l$ in their definitions would almost surely not be atoms of $\mu_{\theta_l}$ and $\nu_{\theta_l}$. Moreover, \Cref{prop:smoothed_estimators_estimate_sw_of_smoothed} would still hold, and its proof would be mostly unchanged (the main difference would be that \eqref{eq:app:l_131} and \eqref{eq:app:l_134} would be replaced by $F_{(\mu_\veps)_\theta}(a) = F_{\mu_\theta}^{(\veps)}(a)$ and $F_{(\nu_\veps)_\theta}(a) = F_{\nu_\theta}^{(\veps)}(a)$, which would allow us to skip the step \eqref{eq:app:l_140} of the proof).
\end{remark}

\begin{remark}
    Our definition in Equation \eqref{eq:app:smoothed_cdf} does not represent the only way of smoothing CDFs. Indeed, we could have used instead of the sigmoid $\sigma$ any monotone function on $\R$ that smoothly varies from 0 to 1. For example, we may have set
    \begin{equation}
        F_\rho^{(\veps)}(a) := \int \Phi(-(x-a-\sqrt{\veps})/\veps) \dd\rho(x)
    \end{equation}
    for every $\rho \in \cP(\R)$ and $a \in \R$, where $\Phi$ is the CDF of the standard Gaussian. In this case, we would have $\SW_{1,\veps}(\mu,\nu) = \SW_1(\mu \ast \rho_\veps, \nu \ast \rho_\veps)$ and $\SW_{2,\veps} = \SW_2(\mu \ast \rho_\veps, \nu \ast \rho_\veps)$ where $\rho = \cN(0,I_d)$ and $\rho_\veps := m_{\veps\#}\rho = \cN(0,\veps^2 I_d)$, so that the smoothed estimators would be (biased) estimators of a Sliced-Wasserstein distance between Gaussian smoothings of the measures. While it may appear more natural to work with Gaussian measures (compared to the elaborate construction of $\rho$ in the proof of \Cref{prop:smoothed_estimators_estimate_sw_of_smoothed}), we nevertheless chose to use the sigmoid $\sigma$ as it is more efficiently computable than $\Phi$.
\end{remark}

\subsection{Experiments}

All experiments in this section were conducted on a system equipped with an Intel Xeon CPU and a NVIDIA T4 GPU.

\subsubsection{Time vs accuracy and trade-off} \label{appendix:sec:sec:sec:exp_time_vs_acc}

\begin{figure}
    \begin{center}
        \includegraphics[width=\linewidth]{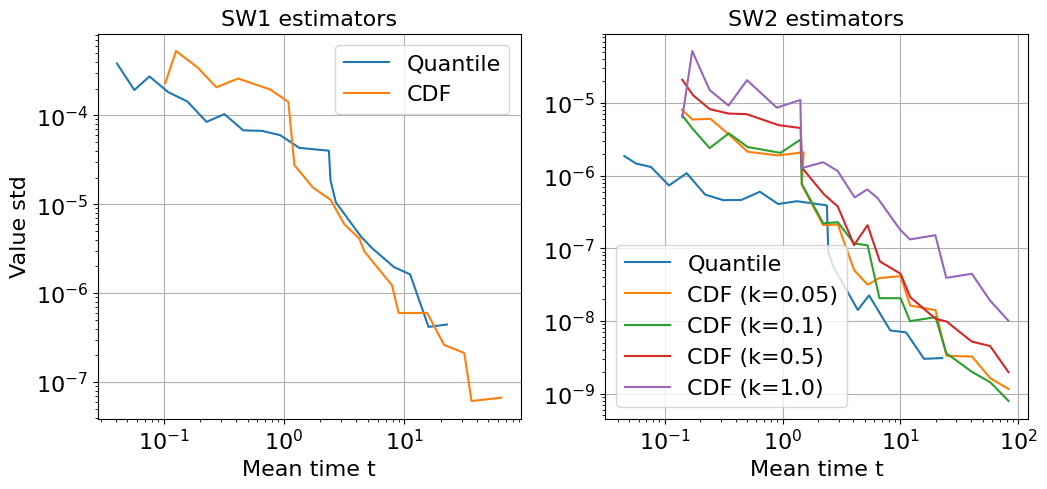}
    \end{center}
    \caption{Accuracy vs computing time for the different $\SW$ estimators. Left: $\SW_1$ estimators, right: $\SW_2$ estimators.}
    \label{fig:time_accuracy}
\end{figure}

In this section we perform a simple experiment comparing the usual quantile function-based estimator and the CDF-based estimators of the $\SW$ distance in terms of the accuracy that can be obtained with a given time budget. We first fix two points clouds $\mu$ and $\nu$ of $N = 10^6$ points each in $d = 10$ dimensions, sampled according to the standard $d$-dimensional Gaussian distribution. Then, for each estimator $\hat{S}_L$ among $\widetilde{\SW}_{1,L}(\mu,\nu)$, $\widetilde{\SW}_{2,L}^2(\mu,\nu)$, $\widehat{\SW}_{1,L}(\mu,\nu)$, and $\widehat{\SW}_{2,L,\eta_k}^2(\mu,\nu)$ for $k \in \{0.05, 0.1, 0.5, 1\}$, for various values of $L$, we compute 10 times $\hat{S}_L(\mu,\nu)$ and derive the standard deviation $\sigma_L$ of the obtained values, while measuring the mean computation time $t_L$ of the estimator, and we draw the graph of the $\sigma_L$ as a function of the $t_L$. The chosen values of $L$ range from 10 to 10000 for the quantile-based estimators, and from 100 to 100000 for the CDF-based estimators. The resulting graphs are shown in \Cref{fig:time_accuracy}. From this experiment we make the following observations:
\begin{itemize}
    \item In the case of the CDF estimators of the $\SW_2$ distance, the estimators $\widehat{\SW}^2_{2,L,\eta_k}$ with $k < 1$ do empirically achieve a better accuracy than the estimator $\widehat{\SW}^2_{2,L,\eta_1} = \widehat{\SW}^2_{2,L}$. In other words, the importance sampling strategy does indeed improve the accuracy of the CDF-based $\SW_2$ estimator.
    \item For a given time budget, the quantile-based estimators generally achieve a better accuracy than the CDF-based estimators. However, the CDF-based $\SW_1$ estimator $\widehat{\SW}_{1,L}$ and the $\SW_2$ estimators $\widehat{\SW}_{2,L,\eta_k}$ for $k \leq 0.1$ achieve comparable accuracy (and sometimes better for $\widehat{\SW}_{1,L}$) to their quantile-based counterparts. 
    \item Since, in \Cref{fig:time_accuracy}, the curves for the quantile-based estimators are plotted for $L$ ranging from 10 to 10000, while those for the CDF-based estimators are plotted for $L$ ranging from 100 to 100000, we infer that a rough rule-of-thumb is that to achieve comparable accuracy, the parameter $L$ must be around one order of magnitude larger for the CDF-based estimators than for the quantile-based estimators.
\end{itemize}

\subsubsection{Tuning the smoothing parameter} \label{appendix:sec:sec:sec:tune_smoothing}

\begin{figure}
    \begin{center}
        \includegraphics[width=\linewidth]{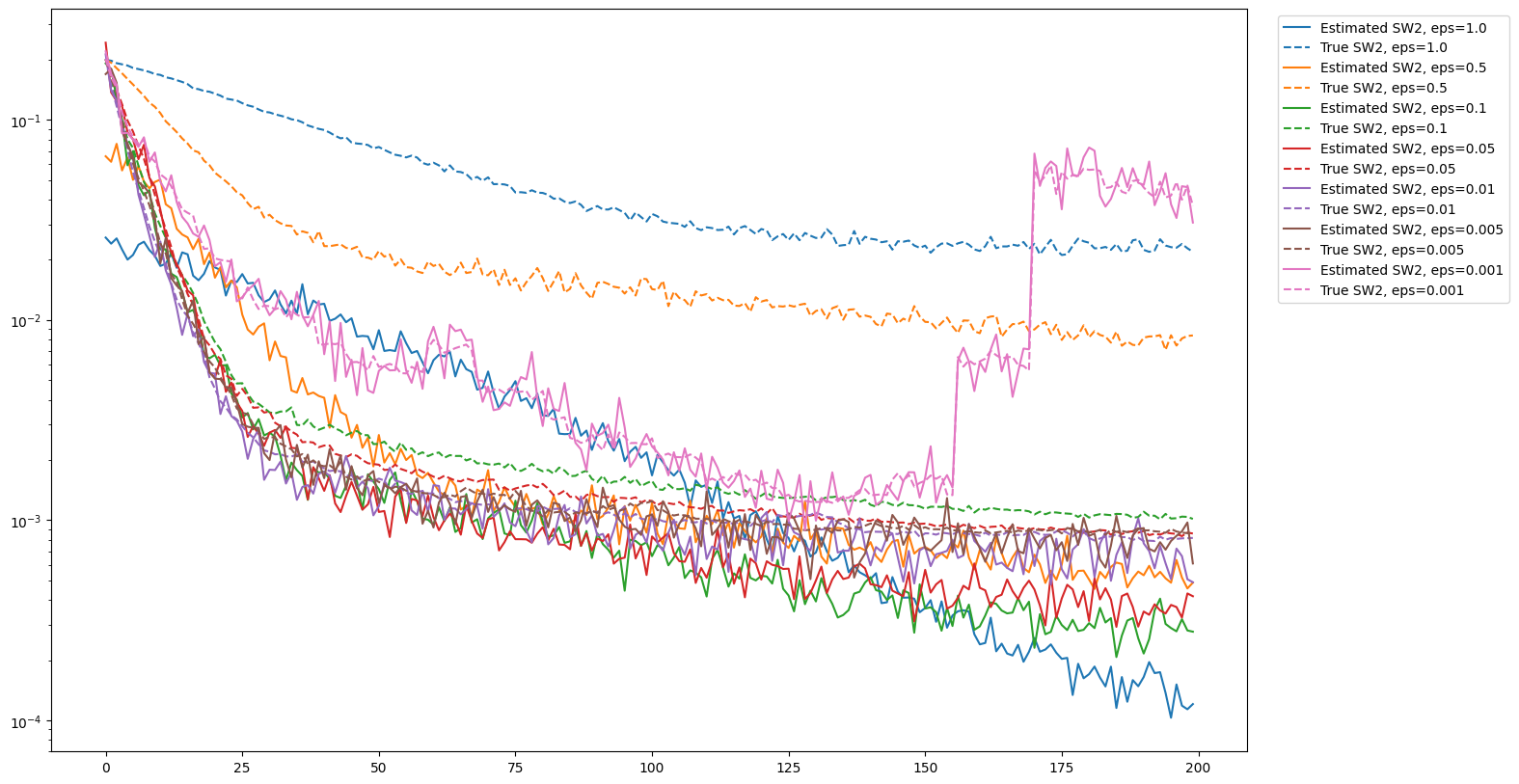}
    \end{center}
    \caption{Gradient descent on the smoothed CDF-based estimators $\widehat{\SW}^2_{2,L,\veps}$ for various values of $\veps$. Solid lines: value of the smoothed CDF-based estimator. Dashed lines: actual value of the SW, as measured by a quantile-based estimator.}
    \label{fig:tune_smoothing}
\end{figure}

Our goal in this experiment is to determine which value to choose for the smoothing parameter $\veps$ in the CDF-based estimator $\widehat{\SW}^2_{2,L,\eta,\veps}$. For this, we fit a point cloud $X$ in the plane to approximate a target probability distribution $\nu$, by performing a gradient descent on the loss function $X \in (\R^2)^N \to \widehat{\SW}^2_{2,L,\eta_k,\veps}(\mu_X, \nu)$, with $\mu_X := \frac 1N \sum_{i=1}^N \delta_{X_i}$, for various values of the parameter $\veps$, ranging from $\veps = 1.0$ to $\veps = 0.001$. The initial point cloud $X_0 \in (\R^2)^N$ is sampled uniformly on the square $[-1,1]^2$, and $\nu$ is the standard Gaussian $\cN(0,I_2)$, which we approximate with $M = 50000$ random samples. We take $N = 500$, $L = 1000$ and $k = 0.05$, and we run the gradient descent for $200$ steps, with a gradient step size of $0.1N$. Moreover, at every step of the gradient descent, we estimate the true value of $\SW_2^2(\mu_X,\nu)$ by computing the quantile estimator $\widetilde{\SW}^2_{2,\tilde{L}}(\mu_X,\nu)$, with $\tilde{L} = 100$. The results are shown in \Cref{fig:tune_smoothing}. We see that for the lowest value of $\veps$, while the CDF-based estimator $\widehat{\SW}^2_{2,L,\eta_k,\veps}$ exhibits low bias (the solid and dashed lines stay close to each other in the graph), the gradient descent is instable and diverges. On the other hand, for the largest values of $\veps$, while the gradient descent does converge, the true value of the SW (as measured by the quantile-based estimator) stays at a high value, so that the smoothed CDF-based estimator exhibits high bias. On the other hand, it appears from this experiment that $\veps \simeq 0.01$ is a good choice for the smoothing parameter of the CDF-based estimator, as it allows both low bias and convergence of the gradient descent converges. 

\section{Differential privacy analysis of the federated Sliced-Wasserstein estimators} \label{appendix:sec:fed_sw}

\subsection{The federated estimator} \label{appendix:sec:sec:fed_sw}

In this section, we analyze the privacy of the federated version of the CDF-based $\SW$ estimators discussed in \Cref{sec:applications}. Consider the setting where we have two datasets $\cD_1 \in (\R^d)^{N_1}$ and $\cD_2 \in (\R^d)^{N_2}$ of respectively $N_1$ and $N_2$ samples, which are split between respectively $C_1$ and $C_2$ client datasets: $\cD_1 = \cD_{1,1} \coprod \ldots \coprod \cD_{1,C_1}$ and $\cD_2 = \cD_{2,1} \coprod \ldots \coprod \cD_{2,C_2}$. We will represent every dataset $\cD = (x_i)_{i=1}^N \in (\R^d)^N$ by the empirical measure of its samples $\mu_\cD := \frac 1N \sum_{i=1}^N \delta_{x_i}$, so that we have for every $i \in \{0,1\}$
\begin{equation}
    \mu_i := \mu_{\cD_i} = \sum_{j=1}^{C_i} \alpha_{i,j} \mu_{i,j}
\end{equation}
where $\mu_{i,j} := \mu_{\cD_{i,j}}$, and $\alpha_{i,j} = N_{i,j}/N_i$ where $N_{i,j}$ is the number of samples of the client dataset $\cD_{i,j}$. This implies in particular that for every $\theta \in \bS^{d-1}$, we have
\begin{equation} \label{eq:app:cdf_is_sum_of_client_cdf}
    F_{\mu_{i,\theta}} = \sum_{j=1}^{C_i} \alpha_{i,j} F_{\mu_{i,j,\theta}}
\end{equation}
\medbreak
We want to compute the $\SW_1$ distance between the datasets $\cD_1$ and $\cD_2$, or more precisely between the probability measures $\mu_1$ and $\mu_2$ representing them, in a federated way. We compute it using the $\widehat{\SW}_{1,L}$ estimator, which we recall is defined for any pair $\mu,\nu \in \cP_2(\R^d)$ as
\begin{equation} \label{eq:app:sw1_cdf_formula_reminder}
    \widehat{\SW}_{1,L}(\mu,\nu) = \frac 1L \sum_{l=1}^L (m_{\theta_l}^+ - m_{\theta_l}^-) |F_{\mu_{\theta_l}}(\tilde{U}_l) - F_{\nu_{\theta_l}}(\tilde{U}_l)|
\end{equation}
where $(\theta_1,U_1)$, ..., $(\theta_l,U_l)$ are i.i.d. random variables with $(\theta_1,U_1) \sim \cU(\bS^{d-1} \times [0,1])$, where $m_\theta^+$ and $m_\theta^-$ are respectively the $\max$ and the $\min$ of $\spt(\mu_\theta) \cup \spt(\nu_\theta)$, and for every $l$ we have noted $\tilde{U}_l := \Rescale(U_l,m_{\theta_l}^-,m_{\theta_l}^+)$. We then compute $\widehat{\SW}_{1,L}(\mu_1,\mu_2)$ using the following procedure:
\begin{enumerate}
    \item \label{item:fed_sw:1} The central server samples $(\theta_1,U_1)$, ..., $(\theta_L,U_L)$ from $\cU(\bS^{d-1} \times [0,1])$, and broadcasts the vector $(\theta_1,U_1,\ldots,\theta_L,U_L)$ to the clients.
    \item \label{item:fed_sw:2} For every $i \in \{1,2\}$, $1 \leq j \leq C_i$, the client holding the dataset $\cD_{i,j}$ computes for every $1 \leq l \leq L$ the bounds $m_{i,j,l}^+ := \max(\spt(\mu_{i,j,\theta_l}))$ and $m_{i,j,l}^- := \min(\spt(\mu_{i,j,\theta_l}))$, and sends the vector $(m_{i,j,1}^-,m_{i,j,1}^+,\ldots,m_{i,j,L}^-,m_{i,j,L}^+)$ to the server.
    \item \label{item:fed_sw:3} For every $1 \leq l \leq L$ the server computes $m_{\theta_l}^\pm$ by taking the maximum of the $m_{i,j,l}^+$ and the minimum of the $m_{i,j,l}^-$. The server then broadcasts the vector $(m_{\theta_1}^-,m_{\theta_1}^+,\ldots,m_{\theta_L}^-,m_{\theta_L}^+)$ to the clients.
    \item \label{item:fed_sw:4} For every $i \in \{1,2\}$, $1 \leq j \leq C_i$, the client holding the dataset $\cD_{i,j}$ computes the CDFs $F_{\mu_{i,j,\theta_1}}(\tilde{U}_1)$, ..., $F_{\mu_{i,j,\theta_L}}(\tilde{U}_L)$ (he can compute the $\tilde{U}_l$ using the $m_{\theta_l}^\pm$ received from the server). It then sends the vector $(F_{\mu_{i,j,\theta_l}}(\tilde{U}_l))_{l=1}^L$ to the server.
    \item \label{item:fed_sw:5} For every $i \in \{1,2\}$, the server computes the CDFs $F_{\mu_{i,l}}(\tilde{U}_l)$, $1 \leq l \leq L$, using the data received from the clients and equation \eqref{eq:app:cdf_is_sum_of_client_cdf} (we assume that the client dataset sizes $N_{i,j}$ are revealed to the server).
    \item \label{item:fed_sw:6} Finally, the server computes $\widehat{\SW}_{1,L}(\mu_1,\mu_2)$ using the CDFs and equation \eqref{eq:app:sw1_cdf_formula_reminder}.
\end{enumerate}

\begin{algorithm}[tb]
  \caption{Federated computation of $\widehat{\SW}_{1,L}$, server side}
  \label{alg:fed_sw1_server}
  \begin{algorithmic}[1]
    \STATE {\bfseries Input:} Parameter $L$, number of clients $C_1, C_2$, client dataset sizes $N_{i,j}$, global dataset sizes $N_1,N_2$
    \STATE Sample $(\theta_1,U_1),\ldots,(\theta_L,U_L)$ uniformly from $\bS^{d-1} \times [0,1]$
    \STATE Broadcast $(\theta_1,U_1,\ldots,\theta_L,U_L)$ to the clients
    \FOR{$i=1,2$}
        \FOR{$j = 1$ {\bfseries to} $C_i$}
            \STATE Receive $(m_{i,j,1}^-,m_{i,j,1}^+,\ldots,m_{i,j,L}^-,m_{i,j,L}^+) \in \R^{2L}$ from client $(i,j)$
        \ENDFOR
    \ENDFOR
    \FOR{$l = 1$ {\bfseries to} $L$}
        \STATE Compute $m_l^- \leftarrow \min_{i,j} m_{i,j,l}^-$ and $m_l^+ \leftarrow \max_{i,j} m_{i,j,l}^+$
    \ENDFOR
    \STATE Broadcast $(m_1^-,m_1^+,\ldots,m_L^-,m_L^+)$ to the clients
    \FOR{$i=1,2$}
        \FOR{$j = 1$ {\bfseries to} $C_i$}
            \STATE Receive $(F_{i,j,1},\ldots,F_{i,j,L}) \in \R^L$ from client $(i,j)$
        \ENDFOR
        \FOR{$l = 1$ {\bfseries to} $L$}
            \STATE Compute $F_{i,l} \leftarrow \sum_{j=1}^{C_i} F_{i,j,l}$
        \ENDFOR
    \ENDFOR
    \STATE Compute $SW \leftarrow \frac 1L \sum_{l=1}^L (m_l^+ - m_l^-) |F_{1,l} - F_{2,l}|$
    \STATE Return $SW$.
  \end{algorithmic}
\end{algorithm}

\begin{algorithm}[tb]
  \caption{Federated computation of $\widehat{\SW}_{1,L}$, client side. Here $i \in \{1,2\}$ and $j \in \{1,\ldots,C_i\}$ denote the index of the client and are fixed.}
  \label{alg:fed_sw1_client}
  \begin{algorithmic}[1]
    \STATE {\bfseries Input:} Private dataset $\mu_{i,j} \in \cP(\R^d)$, private dataset size $N_{i,j}$, global dataset size $N_i$
    \STATE Receive $(\theta_1,U_1),\ldots,(\theta_L,U_L) \in \bS^{d-1} \times [0,1]$ from server
    \FOR{$l = 1$ {\bfseries to} $L$}
        \STATE Compute $m_{i,j,l}^- \leftarrow \min(\spt(\mu_{i,j,\theta_l}))$ and $m_{i,j,l}^+ \leftarrow \max(\spt(\mu_{i,j,\theta_l}))$
    \ENDFOR
    \STATE Send $(m_{i,j,1}^-,m_{i,j,1}^+,\ldots,m_{i,j,L}^-,m_{i,j,L}^+)$ to server
    \STATE Receive $(m_1^-,m_1^+,\ldots,m_L^-,m_L^+) \in \R^{2L}$ from server
    \FOR{$l = 1$ {\bfseries to} $L$}
        \STATE Compute $\tilde{U}_l \leftarrow m_l^- + U_l (m_l^+ - m_l^-)$
        \STATE Compute $F_{i,j,l} \leftarrow \frac{N_{i,j}}{N_i} F_{\mu_{i,j,\theta_l}}(\tilde{U}_l)$
    \ENDFOR
    \STATE Send $(F_{i,j,1},\ldots,F_{i,j,L})$ to server
  \end{algorithmic}
\end{algorithm}

Pseudo-code corresponding to this procedure is presented in \Cref{alg:fed_sw1_server} and \Cref{alg:fed_sw1_client}, for respectively the server-side and the client-side computations\footnote{For simplicity is is assumed that the server knows the client dataset sizes $N_{i,j}$ and that each client knows the size $N_i$ of the global dataset it belongs to}. In the case of the estimator $\widehat{\SW}_{2,L,\eta}$, the procedure is similar. The main differences are that the server samples elements $(\theta_l,U_l,V_l) \sim \cU(\bS^{d-1}) \times \eta$ and broadcasts them to the clients, and that the clients compute and send the CDFs $F_{\mu_{i,j,\theta_l}}(\tilde{U}_l)$ and $F_{\mu_{i,j,\theta_l}}(\tilde{V}_l)$, which the server then use to compute the estimator using equation \eqref{eq:sw2_alt_estimator_2}. Moreover, in the case of the smoothed estimators $\widehat{\SW}_{1,L,\veps}$ and $\widehat{\SW}^2_{2,L,\eta,\veps}$, we proceed likewise, except that the clients compute and the server aggregate the smoothed CDFs $F_{\mu_{i,j,\theta_l}}^{(\veps)}(\tilde{U}_l)$.

\subsection{Privacy issues of the estimator} \label{appendix:sec:sec:fed_sw_privacy_concerns}

This procedure, however, pose a number of concerns in terms of privacy. Indeed, the data that is exchanged between the clients and the server contain information about the clients' datasets that could compromise their privacy. In particular, since a one-dimensional probability measure $\mu \in \cP(\R)$ is determined by its CDF $F_\mu$, and a $d$-dimensional probability measure $\mu \in \cP(\R^d)$ is determined by its projections $\mu_\theta$ (\citep[Proposition 5.1.2]{bonnotte2013unidimensional}), it is in theory possible to reconstruct the client datasets $\cD_{i,j}$ using the CDFs exchanged at step \ref{item:fed_sw:4} of the algorithm provided $L$ is large enough. Fortunately, it is not difficult to significantly mitigate the privacy risks of this algorithm:
\begin{itemize}
    \item First, although the bounds $m_{i,j,l}^\pm$ exchanged at step \ref{item:fed_sw:2} may give information on the dataset $\cD_{i,j}$, it must be noted that if we replace $m_\theta^\pm$ in \eqref{eq:app:sw1_cdf_formula_reminder} by any upper and lower bounds of $\spt(\mu_\theta) \cup \spt(\nu_\theta)$ (instead of the maximum and minimum of this set), we still obtain an unbiased estimator of $\SW_1(\mu,\nu)$ (although with a larger variance). Therefore, the clients can obfuscate the bounds by sending the server $m_{i,j,l}^\pm \pm R_{i,j,l}^\pm$ instead of $m_{i,j,l}^\pm$, where $R_{i,j,l}^\pm$ is some random positive value. Moreover, in some cases, all the measures $\mu_{i,j}$ are supported in some bounded set $B(0,R)$ (for instance, if the samples are images with pixel values ranging from 0 to 1, then all the measures are supported on $[0,1]^d \subseteq B(0,\sqrt{d})$). In such a case, we may as well skip the steps \ref{item:fed_sw:2}-\ref{item:fed_sw:3} and directly use $m_\theta^\pm := \pm R$ for every $\theta \in \bS^{d-1}$.
    \item There exist so-called Secure Aggregation (SecAgg) protocols which allow to compute a sum $\sum_{u \in \cU}^n a_u$ of quantities $a_u$ held by mutually distrustful users $u \in \cU$, which guarantee through cryptographic means that the individual values $a_u$ are not revealed to the other users, as long as some given proportion of the users do not collude with each other (see for example \citep{bonawitz2016practicalsecureaggregation}). Therefore, at steps \ref{item:fed_sw:4}-\ref{item:fed_sw:5} of the algorithm, the clients and server may use such a SecAgg protocol in order to compute the sums $F_{\mu_{i,\theta_l}}(\tilde{U}_l) = \sum_{j=1}^{C_i} \alpha_{i,j} F_{\mu_{i,j,\theta_l}}(\tilde{U}_l)$, $i \in \{1,2\}$, $1 \leq l \leq L$, without revealing the clients' local CDFs $F_{\mu_{i,j,\theta_l}}(\tilde{U}_l)$ to the server and the other clients.
    \item Finally, there is the question of whether the aggregated CDFs $F_{\mu_{i,\theta_l}}(\tilde{U}_l)$ themselves, and any value computed from them, may leak information about the client datasets $\cD_{i,j}$. The main tool in machine learning for tackling such questions is that of \emph{differential privacy}, whose main concepts we will present below. 
\end{itemize}

\subsection{Differential privacy} \label{appendix:sec:sec:differential_privacy}

We present in this section a brief reminder of some of the main results in differential privacy. We refer for further details to the book \citep{dwork2014algorithmic}. Let $\cX$ be a space of datasets, and $\sim$ a binary symmetric relation on $\cX$ denoting "closeness" (for example, a typical choice in the literature is to take $\cX = (\R^d)^N$ the space of datasets with $N$ entries, and for $x,x' \in \cX$ to say that $x \sim x'$ if they differ by at most one entry). Let $\cM : \cX \to \cY$ be a random function, and let $\veps, \delta > 0$.
\begin{definition}{\citep[Definition 2.4]{dwork2014algorithmic}}
    The random function $\cM : \cX \to \cY$ is said to be $(\veps,\delta)$-differentially private if for every pair $x,x' \in \cX$ such that $x \sim x'$, and for every measurable set $S \subseteq \cY$, one has
    \begin{equation} \label{eq:app:eps_dp}
        \bP[\cM(x') \in S] \leq e^\veps \bP[\cM(x) \in S] + \delta.
    \end{equation}
    In the case $\delta = 0$, we say that $\cM$ is $\veps$-differentially private.
\end{definition}
Intuitively, $(\veps,\delta)$-differential privacy means that the randomized algorithm $\cM$ will behave similarly on similar datasets. The notion of $\veps$-differential privacy is stronger than $(\veps,\delta)$-differential privacy: indeed, to say that $\cM$ is $\veps$-differentially private means that for every run $\cM(x)$, the output will be almost equally likely to be observed for a run of $\cM$ on any $x' \sim x$ ; while $(\veps,\delta)$-differential privacy means that for a given pair $x \sim x'$, it is unlikely that a run of $\cM(x)$ will yield an output that is significantly more or less likely to be observed after a run of $\cM(x')$ (see the discussion after \citep[Definition 2.4]{dwork2014algorithmic} for more details). One of the advantages of $(\veps,\delta)$-differential privacy is its ease of manipulation, as it is preserved by any computation:

\begin{proposition}{\citep[Proposition 2.1]{dwork2014algorithmic}} \label{prop:dp_is_preserved}
    Let $\cM : \cX \to \cY$ be an $(\veps,\delta)$-differentially private mechanism, and let $f : \cY \to \cY'$ be a (deterministic or random) function. Then $f \circ \cM : \cX \to \cY'$ is also $(\veps,\delta)$-differentially private.
\end{proposition}

\begin{remark}
    Note that the definitions and results in \citep{dwork2014algorithmic} are given for a specific choice of $\cX$ and $\sim$. However, the proofs work for generic $\cX$, $\sim$. 
\end{remark}

Moreover, different independent $(\veps,\delta)$-differentially private mechanisms can be combined:

\begin{theorem}{\citep[Theorem 3.16]{dwork2014algorithmic}} \label{th:dp_compositionality}
    For every $i = 1,\ldots,k$, let $\cM_i : \cX \to \cY_i$ be a $(\veps_i,\delta_i)$-differentially private random mechanism, such that the $\cM_1,\ldots,\cM_k$ are independent. Then the random mechanism $\cM : \cX \to \Pi_{i=1}^k \cY_i$ defined by $\cM(x) := (\cM_1(x),\ldots,\cM_k(x))$ is $(\veps,\delta)$-differentially private, with $\veps := \sum_{i=1}^k \veps_i$ and $\delta := \sum_{i=1}^k \delta_i$.    
\end{theorem}

Furthermore, given a (deterministic) function $f : \cX \to \R^k$, it is possible, provided some conditions are met, to turn it into an $(\veps,\delta)$-differentially private mechanism for arbitrary $\veps,\delta$ by noising the output of $f$ appropriately. To show this, we first introduce some definitions:

\begin{definition}{\citep[Definition 3.2, Definition 3.8]{dwork2014algorithmic}} \label{def:dp_sensitivity}
    Let $p \geq 1$. Let $\ell_p$-sensitivity of $f : \cX \to \R^k$ is
    \begin{equation}
        \Delta_p(f) := \max_{\substack{x,y \in \cX \\ x \sim y}} |f(x)-f(y)|_p
    \end{equation}
    where $|\cdot|_p$ is the $\ell^p$ norm on $\R^k$.
\end{definition}

Recall also that for every $a \in \R$ and $b > 0$, the Laplace distribution of mean $a$ and scale $b$ is the probability distribution on $\R$ denoted $\Lap(a,b)$ with probability density function 
\begin{equation}
    \Lap(x|a,b) = \frac{1}{2b} \exp\left(-\frac{|x-a|}{b}\right), \quad x \in \R.
\end{equation}
The mean and variance of $\Lap(a,b)$ are respectively $a$ and $2b^2$. We will often note $\Lap(b) := \Lap(0,b)$. We are now ready to construct the Laplace mechanism:

\begin{theorem}{\citep[Theorem 3.3]{dwork2014algorithmic}} \label{th:dp_laplace_mechanism}
    Let $f : \cX \to \R^k$ be some function with $\ell_1$-sensitivity $\Delta_1(f) < +\infty$. Let $\veps > 0$, and define the \emph{Laplace mechanism} as the random function $\cM_{L,\veps} : \cX \to \R^k$ defined by
    \begin{equation}
        \cM_{L,\veps}(x) := f(x) + (Y_1,\ldots,Y_k)
    \end{equation}
    where the $Y_i$ are i.i.d. random variables with law $\Lap(\Delta_1(f)/\veps)$. Then the Laplace mechanism $\cM_{L,\veps}$ is $\veps$-differentially private.
\end{theorem}

One may also noise the output of $f$ using Gaussian noises to obtain an $(\veps,\delta)$-differentially private mechanism:
\begin{theorem}{\citep[Theorem 3.22]{dwork2014algorithmic}} \label{th:dp_gaussian_mechanism}
    Let $f : \cX \to \R^k$ be some function with $\ell_2$-sensitivity $\Delta_2(f) < +\infty$. Let $\veps \in (0,1)$, $\delta > 0$ and $\sigma > 0$, and define the \emph{Gaussian mechanism} as the random function $\cM_{G,\sigma} : \cX \to \R^k$ defined by
    \begin{equation}
        \cM_{G,\sigma}(x) := f(x) + (Y_1,\ldots,Y_k)
    \end{equation}
    where the $Y_i$ are i.i.d. random variables with law $\cN(0,\sigma^2)$. Then, provided that $\sigma \geq \sqrt{2\ln(1.25/\delta)}\Delta_2(f)/\veps$, the Gaussian mechanism $\cM_{G,\sigma}$ is $(\veps,\delta)$-differentially private.
\end{theorem}
Although $(\veps,\delta)$-differential privacy is a weaker property than $\veps$-differential privacy, it may sometimes be desirable to work with Gaussian noises due to their favorable properties.

\subsection{A private federated estimator}

We are now ready to provide modified versions of \Cref{alg:fed_sw1_server} and \Cref{alg:fed_sw1_client} that address the privacy issues raised in \Cref{appendix:sec:sec:fed_sw_privacy_concerns}. As in \Cref{appendix:sec:sec:fed_sw}, the analysis below will focus on the computation of the estimator $\widehat{\SW}_{1,L}$, but it again will be easily adaptable to the $\SW_2$ CDF-based estimators and the smoothed CDF-based estimators. For simplicity, we will assume that the datasets $\cD_1$ and $\cD_2$ have the same number of samples $N = N_1 = N_2$ and are split into the same number of clients $C = C_1 = C_2$, and that each client has the same number of samples $N_{i,j} = N/C$ (so that $\alpha_{i,j} = 1/C$). We also assume that all the measures $\mu_{i,j}$ are supported in some ball $B(0,R)$, so that we may take $m_\theta^\pm = \pm R$ for every $\theta \in \bS^{d-1}$\footnote{In particular, we assume that this information is known to every client}. The modified algorithms are shown in \Cref{alg:fed_sw1_server_private} and \Cref{alg:fed_sw1_client_private}. The main changes are the fact that the knowledge of $R$ allow to eliminate the phase devoted to computing the $m_{\theta_l}^\pm$ ; the use of a secure aggregation protocol for summing the client CDFs ; and the introduction of a noising parameters $\sigma > 0$ at the client side. For every $1 \leq l \leq L$, $i \in \{1,2\}$ and $1 \leq j \leq C$, the client $(i,j)$ sends to the server the noised value $F_{i,j,l} = \frac 1C F_{\mu_{i,j,\theta_l}}(\tilde{U}_l) + Z_{i,j,l}$ where $Z_{i,j,l} \sim \cN(0,\sigma^2)$, so that what the value $F_{i,l}$ the server receives after performing secure aggregation is not the exact value $F_{\mu_{i,\theta_l}}(\tilde{U}_l)$ of the CDF of the global dataset, but a noised value
\begin{equation}
    F_{i,l} = \sum_{j=1}^C \frac 1C F_{\mu_{i,j,\theta_l}}(\tilde{U}_l) + Z_{i,j,l} = F_{\mu_{i,\theta_l}}(\tilde{U}_l) + \sum_{j=1}^C Z_{i,j,l}.
\end{equation}
Moreover, since every $Z_{i,j,l}$ has law $\cN(0,\sigma^2)$, the random variable $\sum_{j=1}^C Z_{i,j,l}$ has law $\cN(0,C\sigma^2)$, that is a centered Gaussian law of standard deviation $\sqrt{C}\sigma$. The vector $(F_{i,1},\ldots,F_{i,L})$ received by the server is thus in practice the result of a Gaussian mechanism applied to the CDF vector $(F_{\mu_{i,\theta_l}}(\tilde{U}_l))_{l=1}^L$. The rest of this section will consist of an analysis of this vector using the framework of differential privacy (\Cref{appendix:sec:sec:differential_privacy}).

\begin{algorithm}[tb]
  \caption{Private federated computation of $\widehat{\SW}_{1,L}$, server side. Here $\mathrm{SecAgg}$ denotes any secure aggregation protocol.}
  \label{alg:fed_sw1_server_private}
  \begin{algorithmic}[1]
    \STATE {\bfseries Input:} Parameter $L$, number of clients per dataset $C$, global dataset size $N$, global dataset support size $R$
    \STATE Sample $(\theta_1,U_1),\ldots,(\theta_L,U_L)$ uniformly from $\bS^{d-1} \times [0,1]$
    \STATE Broadcast $(\theta_1,U_1,\ldots,\theta_L,U_L)$ to the clients
    \FOR{$i=1,2$}
        \FOR{$l = 1$ {\bfseries to} $L$}
            \STATE Compute $F_{i,l} \leftarrow \mathrm{SecAgg}(F_{i,1,l},\ldots,F_{i,C,l})$, where $F_{i,j,l}$ is received from client $(i,j)$
        \ENDFOR
    \ENDFOR
    \STATE Compute $SW \leftarrow \frac 1L \sum_{l=1}^L 2R |F_{1,l} - F_{2,l}|$
    \STATE Return $SW$
  \end{algorithmic}
\end{algorithm}

\begin{algorithm}[tb]
  \caption{Private federated computation of $\widehat{\SW}_{1,L}$, client side. Here $i \in \{1,2\}$ and $j \in \{1,\ldots,C_i\}$ denote the index of the client and are fixed.}
  \label{alg:fed_sw1_client_private}
  \begin{algorithmic}[1]
    \STATE {\bfseries Input:} Private dataset $\mu_{i,j} \in \cP(\R^d)$, number of clients $C$, global dataset size $N_i$, noising parameter $\sigma$, global dataset support size $R$
    \STATE Receive $(\theta_1,U_1),\ldots,(\theta_L,U_L) \in \bS^{d-1} \times [0,1]$ from server
    \FOR{$l = 1$ {\bfseries to} $L$}
        \STATE Compute $\tilde{U}_l \leftarrow -R + 2R U_l$
        \STATE Sample $Z_{i,j,l}$ from $\cN(0,\sigma^2)$
        \STATE Compute $F_{i,j,l} \leftarrow \frac{1}{C} F_{\mu_{i,j,\theta_l}}(\tilde{U}_l) + Z_{i,j,l}$
    \ENDFOR
    \STATE Send $(F_{i,j,1},\ldots,F_{i,j,L})$ to server.
  \end{algorithmic}
\end{algorithm}

For this, we set $\cX := (\R^d)^{N/C \times C \times 2}$ to be the dataset space. An entry $x \in \cX$ is then a tuple $x = (\cD_{i,j})_{i,j}$ containing the client datasets $\cD_{i,j} \in (\R^d)^{N/C}$ for $i=1,2$ and $1 \leq j \leq C$. We will define two closeness relations $\sim_e$ and $\sim_c$. Let $x, x' \in \cX$ with $x = (\cD_{i,j})_{i,j}$ and $x' = (\cD'_{i,j})_{i,j}$, then:
\begin{itemize}
    \item We note $x \sim_c x'$ if there exists at most one pair $(i,j)$ with $i \in \{1,2\}$ and $1 \leq j \leq C$ such that $\cD_{i,j} \neq \cD'_{i,j}$.
    \item We note $x \sim_e x'$ if there exists at most one pair $(i,j)$ such that $\cD_{i,j} \neq \cD'_{i,j}$, and if for such a $(i,j)$ the datasets $\cD_{i,j}$ and $\cD'_{i,j}$ differ by at most one sample.
\end{itemize}
A differentially private mechanism for $\sim_c$ is one where one cannot learn much about individual clients from its output, while a differentially private mechanism for $\sim_e$ is one where one cannot learn much about individual entries (of the global dataset) from its output. One will therefore focus on either depending on the type of protection one wants to achieve.
\medbreak
We thus choose $\sim \in \{\sim_e,\sim_c\}$, and let $\veps,\delta > 0$ be fixed. For simplicity, we focus on a single computation of the estimator $\widehat{\SW}_{1,L}$, and we treat the $(\theta_1,U_1),\ldots,(\theta_L,U_L)$ as fixed values. We are interested in how the noising parameter $\sigma$ of \Cref{alg:fed_sw1_client_private} must be chosen to achieve $(\veps,\delta)$-differential privacy for $\sim$ of the CDF vectors $(F_{i,1},\ldots,F_{i,L})$, $i=1,2$ computed by the server. Indeed, if this is the case, since the noising for the two vectors $i=1,2$ are independent, \Cref{prop:dp_is_preserved} and \Cref{th:dp_compositionality} will ensure that any data computed from them will be $(2\veps,2\delta)$-differentially private, including the $\widehat{\SW}_{1,L}(\mu,\nu)$ estimator. The question of the choice of $\sigma$ is then answered by the following proposition:

\begin{proposition}
    Let $i \in \{1,2\}$, and set $\sigma_0 := \sqrt{2\ln(1.25/\delta)}/\veps$. Then:
    \begin{itemize}
        \item If $\sigma \geq \sigma_0 \frac{\sqrt{L}}{C^{3/2}}$, then the vector $(F_{i,1},\ldots,F_{i,L})$ is $(\veps,\delta)$-differentially private for $\sim_c$.
        \item If $\sigma \geq \sigma_0 \frac{\sqrt{L}}{N\sqrt{C}}$, then the vector $(F_{i,1},\ldots,F_{i,L})$ is $(\veps,\delta)$-differentially private for $\sim_e$.
    \end{itemize}
\end{proposition}

\begin{proof}
    Fix $k \in \{1,2\}$ and let $f_k : \cX \to \R^L$ be the function defined, for every $x \in \cX$ corresponding to the measures $(\mu_{i,j})_{i,j}$, by $f_k(x) := (F_{\mu_{k,\theta_l}}(\tilde{U}_l))_{l=1}^L$. Let $x, x' \in \cX$, with $x = (\cD_{i,j})_{i,j}$ and $x' = (\cD'_{i,j})_{i,j}$, corresponding respectively to the measures $(\mu_{i,j})_{i,j}$ and $(\mu'_{i,j})_{i,j}$.
    \begin{itemize}
        \item If $x \sim_c x'$, then there exists at most one pair $(i,j)$ such that $\cD_{i,j} \neq \cD'_{i,j}$. Since every client CDF $F_{\mu_{i,j,\theta_l}}$ contributes to the global CDF $F_{\mu_{i,\theta_l}}$ by some value in $[0,1/C]$, this implies that $f_k(x)$ and $f_k(x')$ differ by at most $1/C$ at each coordinate. Therefore, the $\ell^2$-sensitivity (\Cref{def:dp_sensitivity}) of $f_k$ (for $\sim_c$) is bounded by
        \begin{equation}
            \Delta_2(f_k) \leq \sqrt{\sum_{l=1}^L \left(\frac 1C\right)^2} = \frac{\sqrt{L}}{C}.
        \end{equation}
        \item If $x \sim_e x'$, then $x$ and $x'$ differ by at most one entry in one client dataset. Since every entry in a dataset $\cD_{i,j}$ contributes to the global CDF $F_{\mu_{i,\theta_l}}$ by either $0$ or $1/N$, this implies that $f_k(x)$ and $f_k(x')$ differ by at most $1/N$ at each coordinate. Therefore, the $\ell^2$-sensitivity of $f_k$ (for $\sim_e$) is bounded by
        \begin{equation}
            \Delta_2(f_k) \leq \sqrt{\sum_{l=1}^L \left(\frac 1N\right)^2} = \frac{\sqrt{L}}{N}.
        \end{equation}
    \end{itemize}
    Since we have seen that the vector $(F_{k,1},\ldots,F_{k,L})$ is a Gaussian mechanism applied to $f_k$ with a Gaussian noise of standard deviation $\sqrt{C}\sigma$, the proposition is thus a direct consequence of \Cref{th:dp_gaussian_mechanism} along with the bounds of the $\ell^2$-sensitivity of $f_k$ for $\sim_c$ and $\sim_e$ derived above.
\end{proof}

\begin{remark}
    The considerations above hold for a single computation of $\widehat{\SW}_{1,L}$, with the slices $(\theta_1,U_1),\ldots,(\theta_L,U_L)$ held fixed. However, it can straightforwardly be adapted to the case of $T$ computations of $\widehat{\SW}_{1,L}$, with $L_1,\ldots,L_T$ slices for each successive computation. Indeed, in this case, we can set $L = L_1 + \ldots + L_T$, and choose $\sigma$ in \Cref{alg:fed_sw1_client_private} so that the vectors $(F_{t,i,l})_{1 \leq t \leq T, 1 \leq l \leq L_t} \in \R^L$, $i=1,2$ (where $F_{t,i,l}$ is the value of $F_{i,l}$ computed by \Cref{alg:fed_sw1_server_private} during its $t$-th execution) are $(\veps,\delta)$-differentially private. In this case, by \Cref{prop:dp_is_preserved} and \Cref{th:dp_compositionality}, any data computed from them will be $(2\veps,2\delta)$-differentially private, including the computed values of $\widehat{\SW}_{1,L_t}(\mu,\nu)$, $t \in \{1,\ldots,T\}$.
\end{remark}

\section{Details on the generative modeling experiments} \label{appendix:sec:details_gans}

The goal of this appendix is to provide additional details on the generative modeling experiments carried out in \Cref{sec:experiments}.

\paragraph{GANs and Sliced-Wasserstein GANs} The architecture we employ is that of Generative Adversarial Networks (GANs). We follow here closely the formalism of \citet{nguyen2022amortized}. Let $\mu \in \cP(\R^d)$ be the data distribution from which we want to be able to sample. We want to train a neural network $G_\phi : \R^{d_z} \to \R^d$ parametrized by $\phi$, called the \emph{generator}, such that $G_{\phi\#}\cN(0,I_{d_z}) = \mu$. That is, our goal is to be able to sample from $\mu$ by sampling $Z \sim \cN(0,I_d)$ and computing $G_\phi(Z)$. In order to train $G_\phi$, we use a second neural network $T_\beta : \R^d \to \R$, called the \emph{discriminator}. During the training, the discriminator learns to differentiate the output of the generator from the true data, while the generator learns to trick the discriminator. Formally, the training of a GAN consists in solving alternatively the following two optimization problems:
\begin{equation} \label{eq:app:gan_disc_loss}
    \min_\beta \E_{x\sim \mu}[\min(0,-1+T_\beta(x))] + \E_{z \sim \cN(0,I_{d_z})}[\min(0,-1-T_\beta(x))]
\end{equation}
\begin{equation} \label{eq:app:gan_gen_loss}
    \min_\phi \E_{z \sim \cN(0,I_{d_z})}[-T_\beta(G_\phi(z))]
\end{equation}
In this experiment, we use instead the variant of GANs known as Sliced-Wasserstein GANs (SW-GANs) \citep{deshpande2018generative}, where the discriminator is of the form $T_\beta = T_{\beta_2} \circ T_{\beta_1}$, with $\beta = (\beta_1,\beta_2)$ and where $T_{\beta_1}$ maps the data space to an intermediate feature space, and where the generator loss \eqref{eq:app:gan_gen_loss} is replaced by
\begin{equation} \label{eq:app:swgan_gen_loss}
    \min_\phi \SW_2^2(\tilde{T}_{\beta_1,\beta_2\#}\mu, \tilde{T}_{\beta_1,\beta_2\#}G_{\phi\#}\cN(0,I_d))
\end{equation}
where $\tilde{T}_{\beta_1,\beta_2}(x) := (T_{\beta_1}(x), T_{\beta_2}(T_{\beta_1}(x)))$ is the concatenation of the output of the discriminator and of the intermediate feature values.

\paragraph{Architecture details} The detail of the architectures that we use for training on MNIST and CelebA are presented respectively in \Cref{tab:mnist_gan_architectures} and \Cref{tab:celeba_gan_architectures}. They are taken and adapted from the code of \citet{nguyen2022amortized}\footnote{Which can be found at \url{https://github.com/UT-Austin-Data-Science-Group/AmortizedSW}}, which itself is an adaptation of the architectures used in \citet{miyato2018spectral} (which introduce the technique of spectral normalization, for stabilizing the training of the discriminator) to the SW-GAN framework of \citet{deshpande2018generative}.

\begin{table}
    \centering
    \caption{GAN architectures for MNIST}
    \begin{tabular}{ccc}
        \begin{tabular}{c}
            \hline
             (a) $G_\phi$ \\ \hline
             Input : $z \in \R^{64} \sim \cN(0,1)$ \\ \hline
             Linear \\
             $64 \to 32 \times 7 \times 7$ \\ \hline
             ResBlock upsample 32 \\ \hline
             ResBlock 32 \\ \hline
             ResBlock upsample 32\\ \hline
             ResBlock 32 \\ \hline
             BatchNorm, ReLU, \\ 
             $3 \times 3$ Conv, Tanh \\ \hline
        \end{tabular}
         &
        \begin{tabular}{c}
            \hline
             (b) $T_{\beta_1}$ \\ \hline
             Input : $x \in [-1,1]^{28 \times 28}$ \\ \hline
             ResBlock downsample 32 \\ \hline
             ResBlock downsample 32\\ \hline
             ResBlock 32\\ \hline
             ResBlock 32\\ \hline
        \end{tabular}
         &
         \begin{tabular}{c}
            \hline
             (c) $T_{\beta_2}$ \\ \hline
             Input : $x \in \R^{32 \times 7 \times 7}$ \\ \hline
             ReLU \\ \hline
             Global sum pooling \\ \hline
             Linear $32 \to 1$ \\
             (spectral normalization) \\ \hline
        \end{tabular}
    \end{tabular}
    \label{tab:mnist_gan_architectures}
\end{table}

\begin{table}
    \centering
    \caption{GAN architectures for CelebA}
    \begin{tabular}{ccc}
        \begin{tabular}{c}
            \hline
             (a) $G_\phi$ \\ \hline
             Input : $z \in \R^{128} \sim \cN(0,1)$ \\ \hline
             Linear \\
             $128 \to 256 \times 4 \times 4$ \\ \hline
             ResBlock upsample 256 \\ \hline
             ResBlock upsample 256 \\ \hline
             ResBlock upsample 256 \\ \hline
             ResBlock upsample 256 \\ \hline
             BatchNorm, ReLU, \\ 
             $3 \times 3$ Conv, Tanh \\ \hline
        \end{tabular}
         &
        \begin{tabular}{c}
            \hline
             (b) $T_{\beta_1}$ \\ \hline
             Input : $x \in [-1,1]^{3 \times 64 \times 64}$ \\ \hline
             ResBlock downsample 128 \\ \hline
             ResBlock downsample 128 \\ \hline
             ResBlock downsample 128 \\ \hline
             ResBlock 128\\ \hline
             ResBlock 128\\ \hline
        \end{tabular}
         &
         \begin{tabular}{c}
            \hline
             (c) $T_{\beta_2}$ \\ \hline
             Input : $x \in \R^{128 \times 8 \times 8}$ \\ \hline
             ReLU \\ \hline
             Global sum pooling \\ \hline
             Linear $128 \to 1$ \\
             (spectral normalization) \\ \hline
        \end{tabular}
    \end{tabular}
    \label{tab:celeba_gan_architectures}
\end{table}

\paragraph{Hyperparameters} For both datasets MNIST and CelebA, we train the models for 100 epochs and with a batch size of 128 (this represents a total of 46900 iterations for MNIST and of 127200 iterations for CelebA). We update the discriminator $T_\beta$ every iteration and the generator $G_\phi$ every 5 iterations. We use the Adam optimizer during the training \citep{kingma2015adam}, with learning rate $0.0002$ and parameters $(\beta_1,\beta_2) = (0,0.9)$\footnote{Those are the parameters of Adam and must not be confused with the parameters of the discriminator $T_{\beta_1}$ and $T_{\beta_2}$!}. The discriminator loss \eqref{eq:app:swgan_gen_loss} is computed using three different estimators of the $\SW_2$ distance:
\begin{itemize}
    \item The usual, quantile-based, estimator $\widetilde{\SW}_{2,L}^2$ with $L = 10000$.
    \item The CDF-based estimator $\widehat{\SW}^2_{2,L,\veps}$ with $L = 100000$ and $\veps = 0.01$.
    \item The CDF-based estimator $\widehat{\SW}^2_{2,L,\eta_k,\veps}$ with $L = 100000$, $k = 0.05$ and $\veps = 0.01$.
\end{itemize}

\paragraph{FID scores} We compute FID scores using the \texttt{pytorch-fid} Python package. Each FID score is computed based on 10000 random samples from the trained model and the full training dataset.

\begin{figure}
    \centering
    \includegraphics[width=\textwidth]{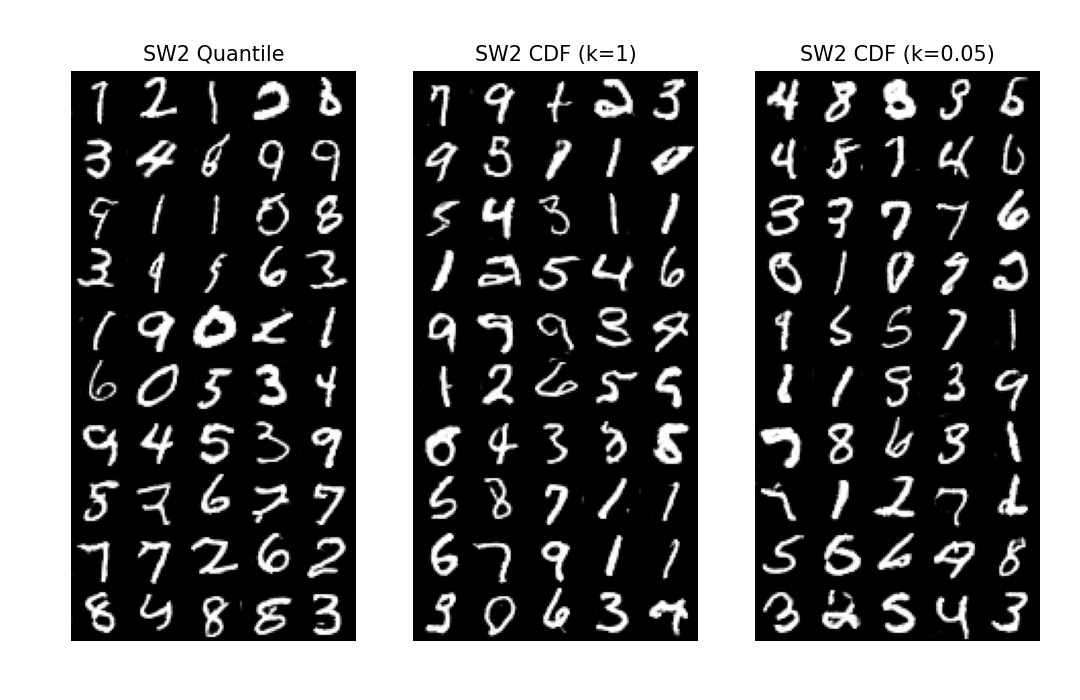}
    \caption{Samples from the GANs trained on MNIST}
    \label{fig:gan_samples_mnist}
\end{figure}

\begin{figure}
    \centering
    \includegraphics[width=\textwidth]{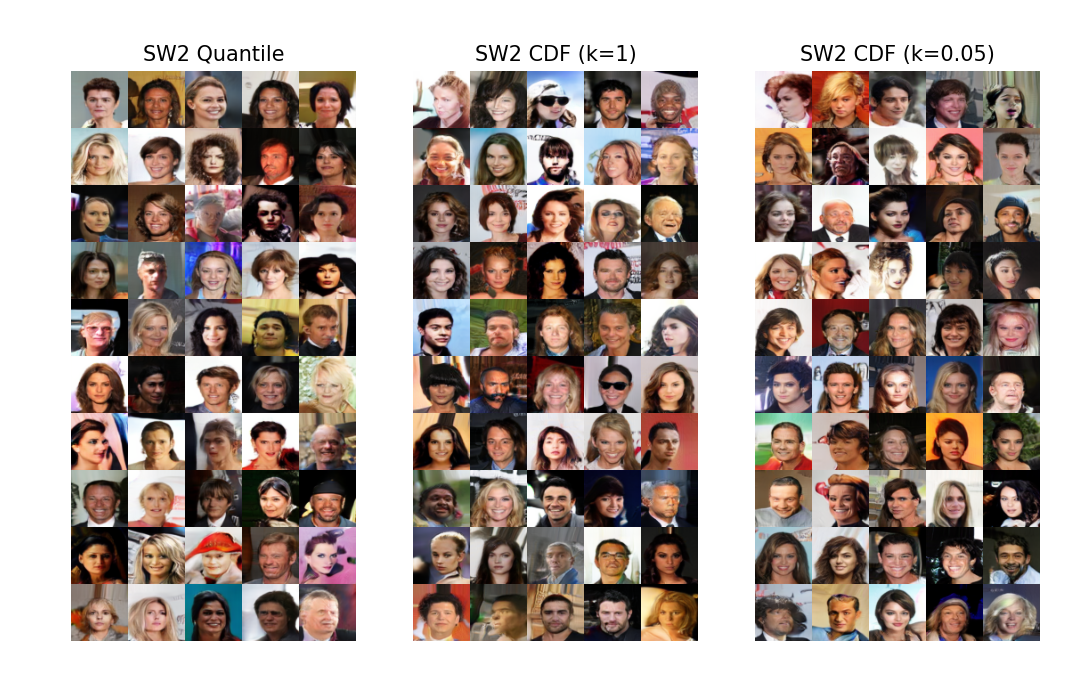}
    \caption{Samples from the GANs trained on CelebA}
    \label{fig:gan_samples_celeba}
\end{figure}

\paragraph{Full results} \Cref{fig:gan_samples_mnist} and \Cref{fig:gan_samples_celeba} show various images sampled from the trained generators for the different datasets and $\SW_2$ estimators.

\section{Details on the federated learning experiments} \label{appendix:sec:details_boosting_fl}

This section provides details on the federated learning algorithm boosting experiments described in \Cref{sec:experiments}.

\subsection{The OTDD and MMDSW dataset distances} \label{appendix:sec:sec:otdd_and_mmdsw}

The Optimal Transport Dataset Distance (OTDD, \citep{alvarez2021dataset}) is a distance between labeled datasets of the form $\cD = \{(x_i,l_i)\}_{i=1}^N \subseteq \R^d \times \{1,\ldots,L\}$. It is defined the following way: for every label $l \in \{1,\ldots,L\}$, let
\begin{equation}
    \nu_{\cD,l} := \frac{1}{\#\{i, x_i = l\}}\sum_{x_i = l} \delta_{x_i} \in \cP(\R^d)
\end{equation}
be the empirical measure of the samples of $\cD$ with label $l$. We then represent the dataset $\cD$ by the measure
\begin{equation}
    \mu_{\cD} := \frac 1N \sum_{i=1}^N \delta_{(x_i,\nu_{\cD,l_i})} \in \cP(\R^d \times \cP(\R^d)),
\end{equation}
and we define the OTDD distance between the labeled datasets $\cD_1$ and $\cD_2$ by
\begin{equation}
    \mathrm{OTDD}^2(\cD_1,\cD_2) := \min_{\gamma \in \Pi(\mu_{\cD_1},\mu_{\cD_2})} \int |x-x'|^2 + \W_2^2(\nu,\nu') \dd\gamma((x,\nu),(x',\nu'))
\end{equation}
where $\Pi(\mu_{\cD_1},\mu_{\cD_2})$ is the set of transport plans between $\mu_{\cD_1}$ and $\mu_{\cD_2}$, that is the set of probability measures $\gamma \in \cP((\R^d \times \cP(\R^d))^2)$ with respective first and second marginals $\mu_{\cD_1}$ and $\mu_{\cD_2}$. The OTDD can thus be computed in a federated manner using the federated algorithm of \citet{rakotomamonjy2024federated} for computing the Wasserstein distance $\W_2$.
\medbreak
Another dataset distance is the Maximum Mean Discrepancy with Riesz $\SW$ kernel (MMDSW, \citep{bonet2025flowing}). Given a labeled dataset $\cD$, we can represent it by the measure
\begin{equation}
    \bP_{\cD} := \frac 1L \sum_{l=1}^L \delta_{\nu_{\cD,l}} \in \cP(\cP(\R^d)).
\end{equation}
Then the MMDSW between two labeled datasets $\cD_1$ and $\cD_2$ is defined as
\begin{equation}
    \mathrm{MMDSW}^2(\cD_1,\cD_2) = \iint K(\mu,\nu)\dd(\bP_{\cD_1}-\bP_{\cD_2})(\mu) \dd(\bP_{\cD_1}-\bP_{\cD_2})(\nu)
\end{equation}
where $K(\mu,\nu) := -\SW_2(\mu,\nu)$ for every $\mu, \nu \in \cP_2(\R^d)$. The MMDSW can therefore be computed in a federated manner using the federated algorithms of \Cref{sec:applications} for computing the Sliced-Wasserstein distance $\SW_2$.

\subsection{The boosting procedure}

Assume then that we are given a labeled dataset $\cD$ split into sub-datasets $\cD_1$, ..., $\cD_C$ among $C$ clients. We wish to train a classifier $T_\theta : \R^d \to \{1,\ldots,L\}$ on this dataset $\cD$ using a federated learning algorithm, but the distributions $\cD_c$ may present some heterogeneity (for instance, some clients may not have samples from all the classes of $\cD$), which may worsen the performance of the trained model. To address this, \citet{rakotomamonjy2024federated} proposed to compute, for every pair $c,c'$ of clients, the OTDD distance $D_{c,c'} := \mathrm{OTDD}(\cD_c,\cD_{c'})$ between their datasets. Then, applying a spectral clustering algorithm (see \citet{vonluxburg2007tutorial}) to the matrix $D \in \R^{C \times C}$, they partition the set of clients into $K$ clusters $C_1$, ..., $C_K$ (with $\{1,\ldots,C\} = C_1 \coprod \ldots \coprod C_k$), and, on every cluster $C_k$, they train one classifier $T_{\theta_k}$ using the federated learning algorithm. Through experiments using various datasets, numbers of clients and clustering and federated learning algorithms, they demonstrate that this approach leads to a consistent increase in the global accuracy of the trained models, which is even more marked when the heterogeneity of the client datasets is increased. We follow the same procedure in our experiment, except that we replace the OTDD with the MMDSW distance.

\subsection{Practical details and results}

\begin{figure}
    \centering
    \includegraphics[width=\textwidth]{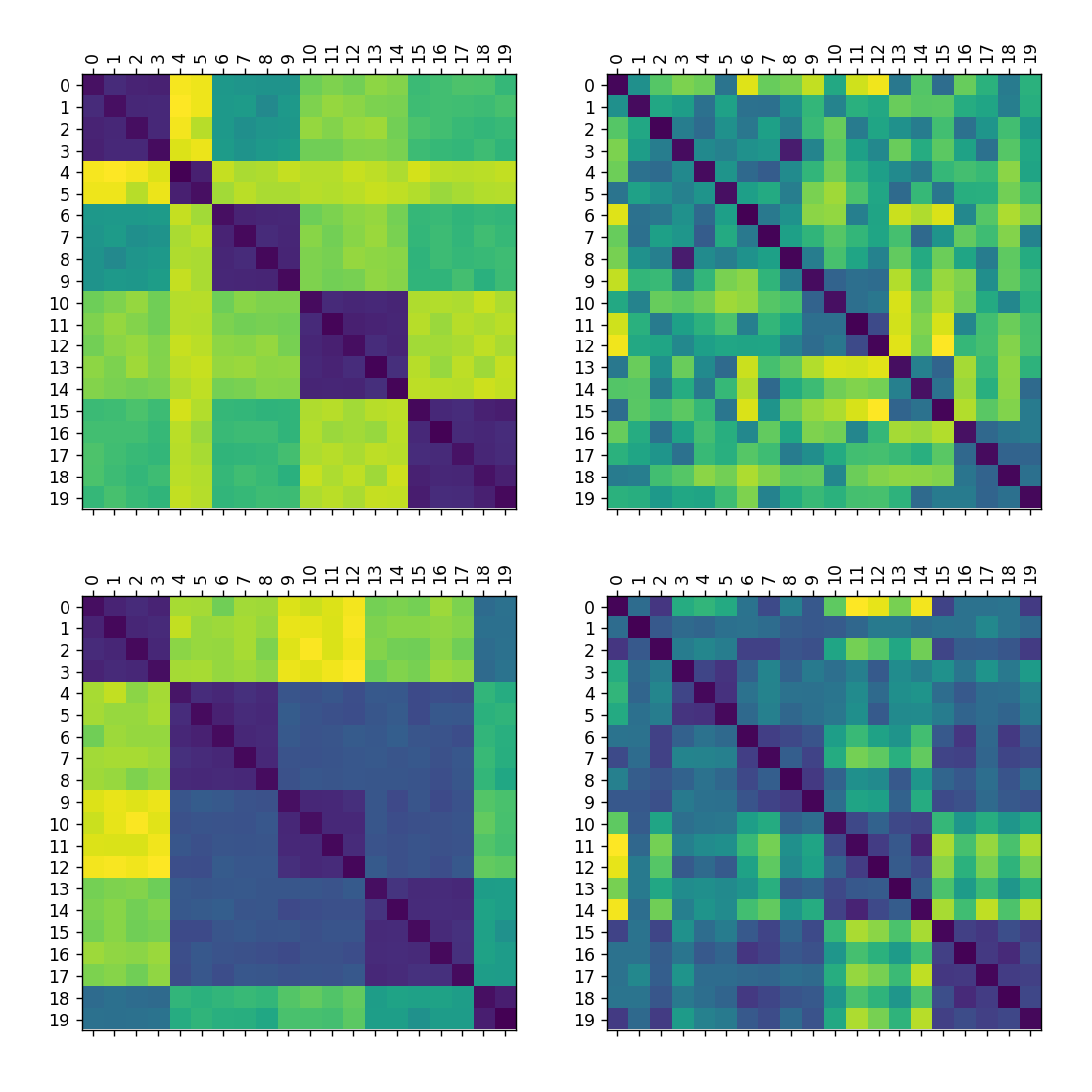}
    \caption{Distance matrices for MMDSW with MNIST (top row) and CIFAR-10 (bottom row), with 20 clients. On the left column we have imposed a cluster structure on the clients, while on the right column we have not imposed a specific structure.}
    \label{fig:dist_matrices}
\end{figure}

\paragraph{Client dataset construction} When we want the client datasets to present a clear cluster structure, we construct them the following way: we assign to each client $c \in \{1,\ldots,C\}$ a random pair $(i_c,j_c)$ of labels among $(0,1),(2,3),(4,5),(6,7),(8,9)$. Then, for every label $l$ of $\cD$, we split the set of samples of $\cD$ with label $l$ equally between the clients to whom $l$ has been assigned. When we \emph{do not} want the client datasets to present a clear cluster structure, we construct them the following way: we split the global dataset $\cD$ into $2C$ shards of similar size, the samples of each shard having the same label, and we randomly assign two shards to each client. \Cref{fig:dist_matrices} show the MMDSW distance matrices obtained after constructing the client datasets for MNIST and CIFAR-10 using these two procedures. We see in particular that when we impose a cluster structure on the clients, this structure is clearly visible on the distance matrix.  

\begin{table}
    \centering
    \caption{Classifier architectures for MNIST (left) and CIFAR-10 (right)}
    \begin{tabular}{cc}
        \begin{tabular}{c}
            \hline
             Input : $x \in [-1,1]^{784}$ \\ \hline
             Linear $784 \to 256$ \\ \hline
             ReLU \\ \hline
             Linear $256 \to 128$ \\ \hline
             ReLU \\ \hline
             Linear $128 \to 32$ \\ \hline
             ReLU \\ \hline
             Linear $32 \to 10$ \\ \hline
             Softmax \\ \hline
        \end{tabular}
         &
        \begin{tabular}{c}
            \hline
             Input : $x \in [-1,1]^{3072}$ \\ \hline
             Linear $3072 \to 512$ \\ \hline
             ReLU \\ \hline
             Linear $512 \to 256$ \\ \hline
             ReLU \\ \hline
             Linear $256 \to 64$ \\ \hline
             ReLU \\ \hline
             Linear $64 \to 10$ \\ \hline
             Softmax \\ \hline
        \end{tabular}
    \end{tabular}
    \label{tab:fl_classifier_architectures}
\end{table}

\paragraph{Classifier architectures} The architectures for the classifiers are shown in \Cref{tab:fl_classifier_architectures}. When using the algorithms FedPer or FedRep, we use the last linear layer of the model as the client's personal head.

\paragraph{Hyperparameters} We compute the MMDSW using the $\widehat{\SW}_{2,L}$ estimator with $L = 10000$. Each classifier is trained during 10 epochs with stochastic gradient descent, with a learning rate of $0.01$ and a momentum of $0.5$, and with a batch size of $10$. The number of clusters is set to $K = 5$.

\paragraph{Supplementary results} Extended results for various settings of the experiment are reported in \Cref{tab:federated_full}. "Vanilla" refers to the setting where we do not perform clustering and train a single classifier on the entire set of clients. In contrast, "Affinity", "Sparse G. (3)" and "Sparse G. (5)" denote three different spectral clustering algorithms: "Affinity" denotes the setting where the MMDSW distance matrix $D$ is directly used (after rescaling) as an affinity matrix by the spectral clustering algorithm (parameter \texttt{affinity="precomputed"} of \texttt{sklearn.SpectralClustering}) while "Sparse G. (3)" and "Sparse G. (5)" denote the setting where $D$ is used to construct a binary matrix of the 3 or 5 nearest neighbors between the clients, which is then used as the affinity matrix for spectral clustering (parameter \texttt{affinity="precomputed\_nearest\_neighbors"}). Note that the spectral clustering algorithm used in \Cref{tab:federated} corresponds to the setting "Affinity".

\begin{table*}
    \small
    \centering
    \caption{Performance for boosted FL algorithms (full results). The average uplifts for OTDD are taken from Tables 2-5 of \citet{rakotomamonjy2024federated} (maximum between support sizes 10 and 100). Note that the higher uplifts for OTDD are in part due to the fact that the baseline performance ("Vanilla") is often lower in \citet{rakotomamonjy2024federated}, so that in our reimplementation the boosting has less room to improve from.}
    
    \resizebox{\linewidth}{!}{
    \begin{tabular}{rcccccccc}
        \hline
         & \multicolumn{4}{c}{Cluster structure} & \multicolumn{4}{c}{No cluster structure} \\ \cline{2-5} \cline{6-9}
         & & \multicolumn{3}{c}{Clustering} & & \multicolumn{3}{c}{Clustering} \\ \cline{3-5} \cline{7-9}
         & Vanilla & Affinity & Sparse G. (3) & Sparse G. (5) & Vanilla & Affinity & Sparse G. (3) & Sparse G. (5) \\ \hline
         & \multicolumn{8}{c}{MNIST} \\ \cline{2-9}
         FedAvg &&&& &&&& \\
         10 & 58.2 $\pm$ 2.3 & \textbf{99.5} $\pm$ \textbf{0.0} & 97.9 $\pm$ 3.2 & 90.8 $\pm$ 4.8 & 66.5 $\pm$ 3.5 & 94.4 $\pm$ 3.1 & \textbf{96.1} $\pm$ \textbf{1.3} & 90.7 $\pm$ 4.2\\
         20 & 58.0 $\pm$ 1.7 & 99.2 $\pm$ 0.1 & 98.6 $\pm$ 1.2 & \textbf{99.2} $\pm$ \textbf{0.1} & 65.6 $\pm$ 5.0 & 84.0 $\pm$ 4.3 & \textbf{85.2} $\pm$ \textbf{3.4} & 85.0 $\pm$ 3.4\\
         40 & 51.7 $\pm$ 3.6 & \textbf{99.0} $\pm$ \textbf{0.0} & 93.8 $\pm$ 4.1 & 99.0 $\pm$ 0.0 & 71.3 $\pm$ 3.7 & 77.0 $\pm$ 3.3 & 76.8 $\pm$ 2.1 & \textbf{77.0} $\pm$ \textbf{1.9}\\
         100 & 56.5 $\pm$ 4.5 & \textbf{98.8} $\pm$ \textbf{0.0} & 96.7 $\pm$ 2.8 & 91.3 $\pm$ 7.5 & 70.3 $\pm$ 2.9 & \textbf{72.5} $\pm$ \textbf{2.3} & 64.8 $\pm$ 2.9 & 70.6 $\pm$ 0.6\\
         FedPer &&&& &&&& \\
         10 & 97.6 $\pm$ 1.4 & \textbf{99.5} $\pm$ \textbf{0.0} & 99.5 $\pm$ 0.0 & 99.3 $\pm$ 0.1 & 96.3 $\pm$ 1.7 & 99.0 $\pm$ 0.1 & \textbf{99.2} $\pm$ \textbf{0.1} & 98.7 $\pm$ 0.5\\
         20 & 98.0 $\pm$ 0.5 & 99.2 $\pm$ 0.0 & 99.2 $\pm$ 0.1 & \textbf{99.2} $\pm$ \textbf{0.1} & 95.7 $\pm$ 1.1 & \textbf{97.9} $\pm$ \textbf{0.3} & 97.9 $\pm$ 0.4 & 97.8 $\pm$ 0.4\\
         40 & 97.5 $\pm$ 0.4 & 99.0 $\pm$ 0.0 & 98.7 $\pm$ 0.3 & \textbf{99.0} $\pm$ \textbf{0.0} & 94.7 $\pm$ 0.9 & 96.5 $\pm$ 0.4 & 96.6 $\pm$ 0.6 & \textbf{97.0} $\pm$ \textbf{0.3}\\
         100 & 95.9 $\pm$ 2.4 & \textbf{98.8} $\pm$ \textbf{0.0} & 98.6 $\pm$ 0.3 & 98.6 $\pm$ 0.2 & 94.7 $\pm$ 1.2 & 96.1 $\pm$ 0.6 & 94.9 $\pm$ 1.2 & \textbf{96.6} $\pm$ \textbf{0.6}\\
         FedRep &&&& &&&& \\
         10 & 98.5 $\pm$ 0.1 & \textbf{98.8} $\pm$ \textbf{0.1} & 98.8 $\pm$ 0.1 & 98.7 $\pm$ 0.1 & 98.1 $\pm$ 0.2 & \textbf{98.7} $\pm$ \textbf{0.1} & \textbf{98.7} $\pm$ \textbf{0.1} & 98.6 $\pm$ 0.2\\
         20 & 98.2 $\pm$ 0.1 & 98.8 $\pm$ 0.1 & \textbf{98.8} $\pm$ \textbf{0.1} & 98.8 $\pm$ 0.1 & 97.3 $\pm$ 0.2 & 97.9 $\pm$ 0.3 & \textbf{98.0} $\pm$ \textbf{0.1} & 97.4 $\pm$ 0.5\\
         40 & 97.4 $\pm$ 0.2 & 98.7 $\pm$ 0.1 & 98.4 $\pm$ 0.4 & \textbf{98.7} $\pm$ \textbf{0.1} & 95.7 $\pm$ 1.2 & 97.3 $\pm$ 0.5 & 97.3 $\pm$ 0.5 & \textbf{97.4} $\pm$ \textbf{0.3}\\
         100 & 92.9 $\pm$ 4.5 & \textbf{98.4} $\pm$ \textbf{0.0} & 98.1 $\pm$ 0.6 & 97.5 $\pm$ 0.9 & 93.8 $\pm$ 0.5 & \textbf{96.1} $\pm$ \textbf{0.5} & 94.9 $\pm$ 1.2 & 95.8 $\pm$ 0.6\\
         \hline
         \multicolumn{2}{l}{Average uplift} &&&&&&& \\
         \multicolumn{2}{l}{MMDSW} & 15.7 $\pm$ 19.8 & 15.2 $\pm$ 19.2 & 14.7 $\pm$ 18.8 & -  & 6.2 $\pm$ 9.7 & 5.3 $\pm$ 11.0 & 5.9 $\pm$ 9.1\\
         \multicolumn{2}{l}{OTDD} & \textbf{29.8} $\pm$ \textbf{28.4} & \textbf{29.0} $\pm$ \textbf{27.2} & \textbf{27.6} $\pm$ \textbf{26.3} & -  & \textbf{18.9} $\pm$ \textbf{18.9} & \textbf{15.7} $\pm$ \textbf{18.6} & \textbf{14.3} $\pm$ \textbf{18.1}\\ \hline
         & \multicolumn{8}{c}{CIFAR-10} \\ \cline{2-9}
         FedAvg &&&& &&&& \\
         10 & 33.5 $\pm$ 1.0 & \textbf{84.8} $\pm$ \textbf{0.2} & 81.2 $\pm$ 7.3 & 65.1 $\pm$ 4.9 & 34.4 $\pm$ 2.0 & \textbf{71.2} $\pm$ \textbf{3.9} & 70.4 $\pm$ 2.4 & 64.0 $\pm$ 4.5\\
         20 & 32.9 $\pm$ 0.9 & \textbf{82.6} $\pm$ \textbf{3.9} & 80.3 $\pm$ 5.4 & 79.4 $\pm$ 6.4 & 35.5 $\pm$ 1.4 & 60.6 $\pm$ 2.2 & \textbf{62.8} $\pm$ \textbf{2.9} & 61.8 $\pm$ 2.7\\
         40 & 31.8 $\pm$ 1.1 & \textbf{84.0} $\pm$ \textbf{0.3} & 81.2 $\pm$ 3.7 & 84.0 $\pm$ 0.3 & 35.2 $\pm$ 1.9 & \textbf{57.5} $\pm$ \textbf{1.8} & 53.9 $\pm$ 2.0 & 57.0 $\pm$ 3.6\\
         100 & 27.2 $\pm$ 2.3 & \textbf{82.4} $\pm$ \textbf{0.2} & 67.5 $\pm$ 9.4 & 74.8 $\pm$ 7.5 & 31.2 $\pm$ 1.1 & \textbf{47.7} $\pm$ \textbf{2.2} & 40.5 $\pm$ 3.3 & 46.1 $\pm$ 2.0\\
         FedPer &&&& &&&& \\
         10 & 82.7 $\pm$ 0.6 & \textbf{84.5} $\pm$ \textbf{0.2} & 84.4 $\pm$ 0.4 & 84.1 $\pm$ 0.3 & 81.6 $\pm$ 3.5 & \textbf{84.0} $\pm$ \textbf{2.8} & 83.9 $\pm$ 2.8 & 84.0 $\pm$ 2.9\\
         20 & 81.6 $\pm$ 0.5 & \textbf{84.7} $\pm$ \textbf{0.2} & 84.5 $\pm$ 0.2 & 84.4 $\pm$ 0.5 & 81.9 $\pm$ 2.5 & 84.2 $\pm$ 2.4 & 84.2 $\pm$ 2.0 & \textbf{84.3} $\pm$ \textbf{2.0}\\
         40 & 80.5 $\pm$ 0.6 & \textbf{84.1} $\pm$ \textbf{0.3} & 84.0 $\pm$ 0.3 & 84.1 $\pm$ 0.2 & 81.7 $\pm$ 1.0 & 82.9 $\pm$ 1.1 & 83.4 $\pm$ 1.1 & \textbf{83.7} $\pm$ \textbf{0.9}\\
         100 & 76.9 $\pm$ 0.8 & \textbf{82.3} $\pm$ \textbf{0.2} & 81.2 $\pm$ 0.5 & 81.8 $\pm$ 0.6 & 77.8 $\pm$ 1.1 & 80.5 $\pm$ 0.7 & 80.2 $\pm$ 1.2 & \textbf{80.8} $\pm$ \textbf{0.7}\\
         FedRep &&&& &&&& \\
         10 & 81.1 $\pm$ 0.3 & \textbf{83.2} $\pm$ \textbf{0.4} & 83.0 $\pm$ 0.6 & 82.6 $\pm$ 0.4 & 80.9 $\pm$ 3.4 & 82.6 $\pm$ 3.0 & 82.9 $\pm$ 2.9 & \textbf{83.1} $\pm$ \textbf{2.7}\\
         20 & 78.7 $\pm$ 0.8 & \textbf{82.6} $\pm$ \textbf{0.2} & 82.5 $\pm$ 0.6 & 82.4 $\pm$ 0.4 & 80.7 $\pm$ 2.6 & \textbf{83.2} $\pm$ \textbf{2.4} & 82.8 $\pm$ 2.2 & 83.1 $\pm$ 2.0\\
         40 & 76.5 $\pm$ 0.6 & \textbf{81.3} $\pm$ \textbf{0.3} & 80.7 $\pm$ 0.5 & 81.2 $\pm$ 0.3 & 78.2 $\pm$ 1.2 & 81.1 $\pm$ 1.1 & 81.2 $\pm$ 0.8 & \textbf{81.5} $\pm$ \textbf{1.0}\\
         100 & 64.2 $\pm$ 1.1 & \textbf{72.9} $\pm$ \textbf{1.0} & 70.0 $\pm$ 1.3 & 70.9 $\pm$ 2.1 & 70.4 $\pm$ 0.8 & 72.9 $\pm$ 0.7 & 72.5 $\pm$ 0.9 & \textbf{73.7} $\pm$ \textbf{0.5}\\
         \hline
         \multicolumn{2}{l}{Average uplift} &&&&&&& \\
         \multicolumn{2}{l}{MMDSW} & \textbf{24.3} $\pm$ \textbf{27.2} & \textbf{22.6} $\pm$ \textbf{26.0} & \textbf{22.4} $\pm$ \textbf{25.6} & -  & 15.7 $\pm$ 18.8 & \textbf{14.5} $\pm$ \textbf{17.9} & \textbf{15.4} $\pm$ \textbf{17.7}\\
         \multicolumn{2}{l}{OTDD} & 20.0 $\pm$ 21.3 & 12.4 $\pm$ 15.1 & 13.3 $\pm$ 16.1 & -  & \textbf{19.9} $\pm$ \textbf{18.0} & 8.6 $\pm$ 15.4 & 9.1 $\pm$ 16.6\\ \hline
    \end{tabular}}
    \label{tab:federated_full}
\end{table*}

%%%%%%%%%%%%%%%%%%%%%%%%%%%%%%%%%%%%%%%%%%%%%%%%%%%%%%%%%%%%

\newpage
\section*{NeurIPS Paper Checklist}

\begin{enumerate}

\item {\bf Claims}
    \item[] Question: Do the main claims made in the abstract and introduction accurately reflect the paper's contributions and scope?
    \item[] Answer: \answerYes{} % Replace by \answerYes{}, \answerNo{}, or \answerNA{}.
    \item[] Justification: The abstract and introduction reflect the paper's contributions and scope.

\item {\bf Limitations}
    \item[] Question: Does the paper discuss the limitations of the work performed by the authors?
    \item[] Answer: \answerYes{} % Replace by \answerYes{}, \answerNo{}, or \answerNA{}.
    \item[] Justification: The limitations are discussed in the conclusion.

\item {\bf Theory assumptions and proofs}
    \item[] Question: For each theoretical result, does the paper provide the full set of assumptions and a complete (and correct) proof?
    \item[] Answer: \answerYes{} % Replace by \answerYes{}, \answerNo{}, or \answerNA{}.
    \item[] Justification: We provide rigorous proofs of all of our theoretical claims in the appendix.

    \item {\bf Experimental result reproducibility}
    \item[] Question: Does the paper fully disclose all the information needed to reproduce the main experimental results of the paper to the extent that it affects the main claims and/or conclusions of the paper (regardless of whether the code and data are provided or not)?
    \item[] Answer: \answerYes{} % Replace by \answerYes{}, \answerNo{}, or \answerNA{}.
    \item[] Justification: All the information needed to reproduce the experiments are disclosed in the paper and its appendix.

\item {\bf Open access to data and code}
    \item[] Question: Does the paper provide open access to the data and code, with sufficient instructions to faithfully reproduce the main experimental results, as described in supplemental material?
    \item[] Answer: \answerYes{} % Replace by \answerYes{}, \answerNo{}, or \answerNA{}.
    \item[] Justification: Our code will be made public.

\item {\bf Experimental setting/details}
    \item[] Question: Does the paper specify all the training and test details (e.g., data splits, hyperparameters, how they were chosen, type of optimizer) necessary to understand the results?
    \item[] Answer: \answerYes{} % Replace by \answerYes{}, \answerNo{}, or \answerNA{}.
    \item[] Justification: All the details necessary to understand the results are described in the paper and its appendix. 

\item {\bf Experiment statistical significance}
    \item[] Question: Does the paper report error bars suitably and correctly defined or other appropriate information about the statistical significance of the experiments?
    \item[] Answer: \answerNo{} % Replace by \answerYes{}, \answerNo{}, or \answerNA{}.
    \item[] Justification: The experiments in our papers do not involve error bars or statistical significance tests. We report standard deviations for the classifier accuracies in the federated learning experiment and for the FID scores in the generative modeling experiment.

\item {\bf Experiments compute resources}
    \item[] Question: For each experiment, does the paper provide sufficient information on the computer resources (type of compute workers, memory, time of execution) needed to reproduce the experiments?
    \item[] Answer: \answerYes{} % Replace by \answerYes{}, \answerNo{}, or \answerNA{}.
    \item[] Justification: In the case of the experiments on synthetic data in the appendix, where execution time is important, we give all the necessary information on the computer resources necessary to reproduce the experiments.
    
\item {\bf Code of ethics}
    \item[] Question: Does the research conducted in the paper conform, in every respect, with the NeurIPS Code of Ethics \url{https://neurips.cc/public/EthicsGuidelines}?
    \item[] Answer: \answerYes{} % Replace by \answerYes{}, \answerNo{}, or \answerNA{}.
    \item[] Justification: The research conducted in the paper is conform to the NeurIPS Code of Ethics.

\item {\bf Broader impacts}
    \item[] Question: Does the paper discuss both potential positive societal impacts and negative societal impacts of the work performed?
    \item[] Answer: \answerNA{} % Replace by \answerYes{}, \answerNo{}, or \answerNA{}.
    \item[] Justification: There are no direct negative societal impact downstream of the work. 
    
\item {\bf Safeguards}
    \item[] Question: Does the paper describe safeguards that have been put in place for responsible release of data or models that have a high risk for misuse (e.g., pre-trained language models, image generators, or scraped datasets)?
    \item[] Answer: \answerNA{} % Replace by \answerYes{}, \answerNo{}, or \answerNA{}.
    \item[] Justification: The paper poses no such risks.

\item {\bf Licenses for existing assets}
    \item[] Question: Are the creators or original owners of assets (e.g., code, data, models), used in the paper, properly credited and are the license and terms of use explicitly mentioned and properly respected?
    \item[] Answer: \answerNA{} % Replace by \answerYes{}, \answerNo{}, or \answerNA{}.
    \item[] Justification: We do not use existing assets in the paper.

\item {\bf New assets}
    \item[] Question: Are new assets introduced in the paper well documented and is the documentation provided alongside the assets?
    \item[] Answer: \answerNA{} % Replace by \answerYes{}, \answerNo{}, or \answerNA{}.
    \item[] Justification: We do not release new assets.

\item {\bf Crowdsourcing and research with human subjects}
    \item[] Question: For crowdsourcing experiments and research with human subjects, does the paper include the full text of instructions given to participants and screenshots, if applicable, as well as details about compensation (if any)? 
    \item[] Answer: \answerNA{} % Replace by \answerYes{}, \answerNo{}, or \answerNA{}.
    \item[] Justification: Our paper does not involve such research.

\item {\bf Institutional review board (IRB) approvals or equivalent for research with human subjects}
    \item[] Question: Does the paper describe potential risks incurred by study participants, whether such risks were disclosed to the subjects, and whether Institutional Review Board (IRB) approvals (or an equivalent approval/review based on the requirements of your country or institution) were obtained?
    \item[] Answer: \answerNA{} % Replace by \answerYes{}, \answerNo{}, or \answerNA{}.
    \item[] Justification:  Our paper does not involve such research.

\item {\bf Declaration of LLM usage}
    \item[] Question: Does the paper describe the usage of LLMs if it is an important, original, or non-standard component of the core methods in this research? Note that if the LLM is used only for writing, editing, or formatting purposes and does \emph{not} impact the core methodology, scientific rigor, or originality of the research, declaration is not required.
    %this research? 
    \item[] Answer: \answerNA{} % Replace by \answerYes{}, \answerNo{}, or \answerNA{}.
    \item[] Justification: The core method development in this paper does not involve LLMs as any important, original, or non-standard components.

\end{enumerate}

\end{document}